\pdfoutput=1

\newif\ifuc
\ucfalse % exclude parts under construction

\newif\ifarxiv
\arxivfalse % for NeurIPS

\ifarxiv
\documentclass[a4paper]{article}
\else
\documentclass[usletter]{article} % NeurIPS
\fi

\ifarxiv
\usepackage{lmodern} % scalable font 
\else
\usepackage[final]{neurips_2020}
\fi

\usepackage{amsfonts}       % blackboard math symbols
\usepackage{nicefrac}       % compact symbols for 1/2, etc.
\usepackage{microtype}      % microtypography

\usepackage{geometry}
\usepackage[T1]{fontenc}
\usepackage{url}
\usepackage[hidelinks]{hyperref}
\usepackage[small]{caption}
\usepackage{xcolor}
\usepackage{graphicx}
\usepackage{amsmath}
\usepackage{amsthm}
\usepackage{amssymb}
\usepackage{booktabs} % for tables
\usepackage{algorithm}
\usepackage{natbib}
\usepackage{listings} % http://users.ecs.soton.ac.uk/srg/softwaretools/document/start/listings.pdf
\usepackage{mathtools}
\usepackage{appendix}
\usepackage[capitalise]{cleveref}
\usepackage[mathletters]{ucs} % allows some more UniCode math symbols $ω=∇×u$
\usepackage[utf8x]{inputenc} % interprets many UniCode characters correctly $∃α∀β:α→β,∈ℝ,∫_ε^θ≤∞,∑_γ^δ,←λ,x∈A⊂ℝ$,↔,⇔
\usepackage{stmaryrd}
\usepackage{chngcntr}

\theoremstyle{definition}
\newtheorem{theorem}{Theorem}
\newtheorem{lemma}[theorem]{Lemma}
\newtheorem{corollary}[theorem]{Corollary}

\newtheorem{remark}[theorem]{Remark}
\newtheorem{examplex}[theorem]{Example}
\newenvironment{example}
  {\pushQED{\qed}\examplex}
  {\popQED\endexamplex}

\definecolor{darkred}{rgb}{.7,0,0}
\definecolor{darkgreen}{rgb}{0,.5,0}
\definecolor{darkblue}{rgb}{0,0,.8}
\definecolor{darkcyan}{rgb}{0,0.6,.6}
\definecolor{darkorange}{rgb}{.8,.4,0}
\definecolor{gray}{rgb}{.4,.4,.4}

\ifuc % if under construction (set \uctrue or \ucfalse in main.tex lines 8
\newcommand{\todo}[1]{\textcolor{darkorange}{(\emph{TODO: #1})}}
\newcommand{\comment}[1]{\textcolor{darkblue}{(\emph{#1})}}
\newcommand{\warning}[1]{\textcolor{red}{(\emph{WARNING: #1})}}
\newcommand{\quest}[1]{\textcolor{darkgreen}{(\emph{Q: #1})}}
\newcommand{\XXX}{\textcolor{red}{\textbf{XXX}}}
\else
\newcommand{\todo}[1]{}
\newcommand{\comment}[1]{}
\newcommand{\warning}[1]{}
\newcommand{\quest}[1]{}
\newcommand{\XXX}{}
\fi

\newcommand{\Naturals}{\mathbb{N}}

\newcommand{\Reals}{\mathbb{R}}

\newcommand{\ie}{\emph{i.e.}, }
\newcommand{\eg}{\emph{e.g.}, }

\newcommand{\eps}{\varepsilon}

\newcommand{\indicator}[1]{\llbracket #1 \rrbracket}
\newcommand{\ceiling}[1]{\left\lceil#1\right\rceil}

\newcommand{\sgn}{\text{sgn}}

\newcounter{alphoversetcount}
\newcommand{\alphnextref}{\stepcounter{alphoversetcount}\text{(\alph{alphoversetcount})}}
\newcommand{\alphoverset}[1]{\overset{\alphnextref{}}{#1}}
\newcommand{\resetalph}{\setcounter{alphoversetcount}{0}}

\definecolor{dkblue}{cmyk}{1,.54,.04,.19}
\definecolor{dkgray}{rgb}{0.3,0.3,0.3}
\hypersetup{
      bookmarks=true,         % show bookmarks bar?
      unicode=false,          % non-Latin characters in AcrobatÕs bookmarks
      pdftoolbar=true,        % show AcrobatÕs toolbar?
      pdfmenubar=true,        % show AcrobatÕs menu?
      pdffitwindow=false,     % window fit to page when opened
      pdfstartview={FitH},    % fits the width of the page to the window
      pdfcreator={pdflatex},   % creator of the document
      pdfproducer={Producer}, % producer of the document
      pdfnewwindow=true,      % links in new window
      colorlinks=true,        % false: boxed links; true: colored links
      linkcolor=dkgray,       % color of internal links
      citecolor=dkgray,       % color of links to bibliography
      filecolor=dkblue,       % color of file links
      urlcolor=dkblue,        % color of external links
}
\def\v{\boldsymbol}             % boldface vector (alt:\boldsymbol\frak\Bbb\pmb\text)

\newcommand{\nlayers}{\ell}
\newcommand{\wmax}{w_{\text{max}}}
\newcommand{\activ}{\sigma}
\newcommand{\activvec}{\v{\sigma}}
\newcommand{\nvec}{\mathbf{n}}
\newcommand{\inputset}{\mathcal{X}}
\newcommand{\xmax}{x_{\text{max}}}
\newcommand{\Fmax}{F_{\text{max}}}
\newcommand{\nmax}{n_{\text{max}}}
\newcommand{\lambdamax}{\lambda_{\text{max}}}

\newcommand{\msim}{m}%{{\tilde{m}}} % m, for sampled weights
\newcommand{\Msim}{M} % #samples per layer

\newcommand{\mask}{b}
\newcommand{\vmask}{\boldsymbol{\mask}}

\newcommand{\archi}{A(\nlayers, \nvec, \activvec)}
\newcommand{\Fnn}{F}%{F^*}
\newcommand{\fnn}{f}%{f^*} % function of a neuron
\newcommand{\WF}{W^*}%{W^{\scriptscriptstyle F}} % target weight matrix
\newcommand{\wF}{w^*}%{w^{\scriptscriptstyle F}} % target weight

\newcommand{\NF}{N^*} % number of weights in F

\title{Logarithmic Pruning is All You Need}
\ifarxiv
\author{%
Laurent Orseau
\and
Marcus Hutter
\and
Omar Rivasplata
}
\date{
\texttt{\{lorseau,mhutter,rivasplata\}@google.com} \\
DeepMind, London, UK \\
June 11, 2020
}
\else
\author{%
Laurent Orseau \\ DeepMind, London, UK \\ \texttt{lorseau@google.com}
\And
Marcus Hutter \\ DeepMind, London, UK \\ \texttt{www.hutter1.net}
\And
Omar Rivasplata \\ DeepMind, London, UK \\ \texttt{rivasplata@google.com}
}
\fi

\begin{document}
\maketitle
\begin{abstract}
The Lottery Ticket Hypothesis is a conjecture that every large neural network contains a subnetwork that, when trained in isolation, achieves comparable performance to the large network.
An even stronger conjecture has been proven recently: Every sufficiently overparameterized network contains a subnetwork that, at random initialization, but without training, achieves comparable accuracy to the trained large network.
This latter result, however, relies on a number of strong assumptions and guarantees a polynomial factor on the size of the large network compared to the target function.
In this work, we remove the most limiting assumptions of this previous work while providing significantly tighter bounds: 
the overparameterized network only needs a logarithmic factor (in all variables but depth) number of neurons per weight of the target subnetwork.
\end{abstract}

\comment{Keywords: neural networks, lottery ticket hypothesis, pruning, 
NeurIPS: Deep Learning -> Analysis and Understanding of Deep Networks, ...}

\todo{Global ToDos: See more in text:
For submission:
\\
For final accepted version:
\begin{itemize}
    %\item M: Replace 2/3 by Golden Ratio everywhere (since we are so proud of it :-)
%[L: I'm not entirely sure, it's more complicated to read and interpret.] [M: Ok, maybe]
    %\item Possibly overhaul notation. 
    \item M: Standard error propagation of bounded Lipschitz functions should give OverallError ≤  $ε×$ Number of Operations, which is $\ell×n^2$ in our case,
so maybe only $\log(\ell)$ for logistic activation?
    \item The results may hold even if the target network is not fully connected if carefully expressing everying in terms of number of weights per layer.
    \item Remove $k$ or $k'$ in place of the other. They differ only by $≤1$.
    At least inverse the names as $k'$ is currently used more often.
    (also note that $k= \ceiling{k-1}$)
    %\item Update discussion on $\Fmax$: exponential with the depth in general, but inside the log.
\end{itemize}}

%%%%%%%%%%%%%%%%%%%%%%%%%%%%%%%%%%%%%%%%%%%%%%%%%%%%%%%%%%%%%%%
\section{Introduction}
%%%%%%%%%%%%%%%%%%%%%%%%%%%%%%%%%%%%%%%%%%%%%%%%%%%%%%%%%%%%%%%

\warning{Binary mask vector: Check if correct usage! (vector or matrix?)}

The recent success of neural network (NN) models in a variety of tasks, ranging from vision~\citep{khan2020convnet} to speech synthesis~\citep{oord2016wavenet} to playing games~\citep{schrittwieser2019mastering,ebendt2009weighted}, has sparked a number of works aiming to understand how and why they work so well.
% Proving theoretical properties for such complex, high-dimensional dynamical system is quite a difficult task.
Proving theoretical properties for neural networks is quite a difficult task, with challenges due to the intricate composition of the functions they implement and the high-dimensional regimes of their training dynamics. 
The field is vibrant but still in its infancy, many theoretical tools are yet to be built to provide guarantees on what and how NNs can learn.
A lot of progress has been made towards understanding the convergence properties of NNs (see \eg \cite{allen-zhu2019convergence}, \cite{zou2019improved} and references therein).
The fact remains that training and deploying deep NNs has a large cost \citep{livni2014}, which is problematic.
To avoid this problem, one could stick to a small network size.
However, it is becoming evident that there are benefits to using oversized networks, as the literature on overparametrized models \citep{belkin-late2018power} points out.
Another solution, commonly used in practice, is to prune a trained network to reduce the size and hence the cost of prediction/deployment. While missing theoretical guarantees, experimental works show that pruning can considerably reduce the network size without sacrificing accuracy.

The influential work of \citet{frankle2018lottery} has pointed out the following observation:
a) train a large network for long enough and observe its performance on a dataset,
b) prune it substantially to reveal a much smaller subnetwork with good (or better) performance,
c) reset the weights of the subnetwork to their original values and remove the rest of the weights, and
d) retrain the subnetwork in isolation; then the subnetwork reaches the same test performance as the large network, and trains faster.
\citet{frankle2018lottery} thus conjecture that every successfully trained network contains a much smaller subnetwork that, when trained in isolation, has comparable performance to the large network, without sacrificing computing time.
They name this phenomenon the Lottery Ticket Hypothesis, and a `winning ticket' 
is a subnetwork of the kind just described.

\citet{ramanujan2019s} went even further by observing that if the network architecture is large enough,
then it contains a smaller network that, \emph{even without any training}, has comparable accuracy to the trained large network.
They support their claim with empirical results using a new pruning algorithm,
and even provide a simple asymptotic justification
that we rephrase here:
Starting from the inputs and progressing toward the outputs,
for any neuron of the target network, sample as many neurons as required until one 
calculates a function within small error of the target neuron;
then, after pruning the unnecessary neurons, the newly generated network will be within some small error of the target network. Interestingly, \cite{ulyanov2018} pointed out that randomly initialized but untrained ConvNets already encode a great deal of the image statistics required for restoration tasks such as de-noising and inpainting, and the only prior information needed to do them well seems to be contained in the network structure itself, since no part of the network was learned from data. 

Very recently, building upon the work of \citet{ramanujan2019s},
\citet{malach2020proving} proved a significantly stronger version of 
the ``pruning is all you need'' conjecture,
moving away from asymptotic results to non-asymptotic ones:
With high probability, any target network of $\nlayers$ layers and $n$ neurons per layer
can be approximated within $\eps$ accuracy
% by pruning a network that contains $O(\XXX)$ many more weights than the target network.
by pruning a larger network whose size is polynomial in the size of the target network.
To prove their bounds,
\citet{malach2020proving} make assumptions about the norms of the inputs and of the weights.
This polynomial bound already tells us that unpruned networks contain many `winning tickets' even
without training. Then it is natural to ask: could the most important task of gradient descent be pruning?

%Providing theoretical guarantees is % not just a difficult 
%a useful exercise: it allows to obtain a deeper understanding about what really is constraining and may prevent correct functioning in practice,
%and about what does not matter so much.
%For this reason, it is essential to aim at developing sound and precise theoretical guarantees, so as to identify the roles and constraints of the various moving pieces.
%
%Of course our work by no means pretends to explain everything NNs can do, only to provide important insights.
%
% Building on top of these previous works, we aim at providing stronger theoretical guarantees if one insists on the principle that ``pruning is all you need,'' hoping to provide further insights into how such winning tickets may be found.
%
Building on top of these previous works, we aim at providing stronger theoretical guarantees still based on the motto that ``pruning is all you need'' but hoping to provide further insights into how `winning tickets' may be found.
In this work we relax the aforementioned assumptions while greatly strengthening the theoretical guarantees by improving from polynomial to logarithmic order in all variables except the depth,
for the number of samples required to approximate one target weight.

\paragraph{How this paper is organized.}
After some notation (\cref{sec:notations}) and the description of the problem (\cref{sec:obj}),
we provide a general approximation propagation lemma (\cref{sec:approx_propag}),
which shows the effect of the different variables on the required accuracy.
Next, we show how to construct the large, fully-connected ReLU network in \cref{sec:constructG}
identical to \citet{malach2020proving}, except that weights are sampled
from a hyperbolic weight distribution instead of a uniform one.
We then give our theoretical results in \cref{sec:relu},
showing that only $\tilde{O}(\log(\nlayers\nmax/\eps))$ neurons per target weight
are required under some similar conditions as the previous work
(with $\nlayers$ layers, $\nmax$ neurons per layer and $\eps$ accuracy)
or $\tilde{O}(\nlayers\log(\nmax/\eps))$
(with some other dependencies inside the log)
if these conditions are relaxed.
For completeness, the most important technical result is included in \cref{sec:random_continuous}.
Other technical results, a table of notation, and further ideas can be found in \cref{sec:technical}.
%We conclude by discussing some possible areas of future improvement,
%which may reveal what neural networks can do that we are not taking advantage of.

%%%%%%%%%%%%%%%%%%%%%%%%%%%%%%%%%%%%%%%%%%%%%%%%%%%%%%%%%%%%%%%
\section{Notation and definitions}\label{sec:notations}
%%%%%%%%%%%%%%%%%%%%%%%%%%%%%%%%%%%%%%%%%%%%%%%%%%%%%%%%%%%%%%%

A network architecture $\archi$ is described by a positive integer $\nlayers$ corresponding to the number of fully connected feed-forward layers, and a list of positive integers $\nvec = (n_0,n_1,\ldots,n_\nlayers)$ corresponding to the profile of widths, where $n_i$ is the number of neurons in layer $i \in [\nlayers] = \{ 1,\ldots,\nlayers \}$ and $n_0$ is the input dimension, and a list of activation functions $\activvec = (\activ_1, \ldots, \activ_\nlayers)$---all neurons in layer $i$ use the activation function $\activ_i$.
Networks from the architecture $\archi$ implement functions from $\Reals^{n_0}$ to $\Reals^{n_\nlayers}$ that are obtained by successive compositions:
$%\begin{displaymath}
\Reals^{n_0} \longrightarrow \Reals^{n_1}
\longrightarrow \ \cdots %\ \longrightarrow \Reals^{n_{\nlayers-1}}
\longrightarrow \Reals^{n_\nlayers}\,.
$%\end{displaymath}

% Target network $\Fnn$,
% with $\nlayers$ layers,
% $n_i$ is the number of neurons in layer $i$, 
% where $n_0$ is the number of inputs,
% and $n_\nlayers$ is the number of outputs of the network.

Let $\Fnn$ be a \emph{target network} from architecture $\archi$. The composition of such $\Fnn$ is as follows:
Each layer $i\in[\nlayers]$ has a matrix $\WF_i \in [-\wmax, \wmax]^{n_i\times n_{i-1}}$ of connection weights,
and an activation function $\activ_i$,
such as tanh, the logistic sigmoid, ReLU, Heaviside, etc.
The network takes as input a vector $x\in\inputset \subset \Reals^{n_0}$ where for example
$\inputset=\{-1,1\}^{n_0}$ or $\inputset=[0, \xmax]^{n_0}$, etc. %, or even $\inputset=[0, \xmax]^{n_0}$ as for RGB codes.
In layer $i$, the neuron $j$
with in-coming weights $\WF_{i, j}$ calculates 
$\fnn_{i,j}(y) =\activ_i(\WF_{i,j} y)$,
where $y\in\Reals^{n_{i-1}}$ is usually the output of the previous layer. 
Note that $\WF_{i, j}$ is the $j$-th row of the matrix $\WF_{i}$.
The vector ${\fnn_i}(y) = [\fnn_{i,1}(y), \ldots, \fnn_{i,n_i}(y)]^\top \in\Reals^{n_i}$ denotes the output of the whole layer $i$ when it receives $y \in \Reals^{n_{i-1}}$ from the previous layer. 
Furthermore, for a given \emph{network input} $x\in\inputset$
we recursively define $\Fnn_{i}(x)$ by setting $\Fnn_0(x) = x$, and for $i \in [l]$ then $\Fnn_{i}(x) = {\fnn_i}(\Fnn_{i-1}(x))$. 
The output of neuron $j \in [n_{i}]$ in layer $i$
% given network input $x$ 
is
$\Fnn_{i,j}(x) = \fnn_{i, j}(\Fnn_{i-1}(x))$.
The \emph{network output} is $\Fnn(x) = \Fnn_\nlayers(x)$.

For an activation function $\activ(.)$, let $\lambda$ be its Lipschitz factor (when it exists), 
that is, $\lambda$ is the smallest real number such that $|\activ(x) - \activ(y)| \leq \lambda|x-y|$ for all $(x, y) \in \Reals^2$.
For ReLU and tanh we have $\lambda=1$, and for the logistic sigmoid, $\lambda=1/4$.
Let $\lambda_i$ be the $\lambda$ corresponding to the activation function $\activ_i$ of all the neurons in layer $i$,
and let $\lambdamax = \max_{i\in[\nlayers]}\lambda_i$.

%Define $\nmax = \max_{i\in[1..\nlayers]} n_i$ to be the maximum number of neurons of any layer apart from the input layer.
Define $\nmax = \max_{i\in[0..\nlayers]} n_i$ to be the maximum number of neurons per layer.
The total number of connection weights in the architecture $\archi$ is denoted $\NF$, and we have $\NF\leq \nlayers\nmax^2$.

For all $x\in\inputset$, let $\Fmax(x) = \max_{i\in[\nlayers]} \max_{j\in[n_{i-1}]} |\Fnn_{i-1, j}(x)|$ be the maximum activation at any layer of a target network $\Fnn$, including the network inputs but excluding the network outputs.
We also write $\Fmax(\inputset) = \sup_{x\in\inputset}\Fmax(x)$;
when $\inputset$ is restricted to the set of inputs of interest (not necessarily the set of all possible inputs) such as a particular dataset, 
we expect $\Fmax(\inputset)$ to be bounded by a small constant in most if not all cases.
For example, $\Fmax(\inputset)\leq1$ for a neural network with only sigmoid activations and inputs in $[-1, 1]^{n_0}$. 
For ReLU activations, $\Fmax(\inputset)$ can in principle grow as fast as $(\nmax\wmax)^\nlayers$, but since networks with sigmoid activations are universal approximators, we expect that for all functions that can be approximated with a sigmoid network there is a ReLU network calculating the same function with $\Fmax(\inputset) = O(1)$.

The \emph{large network} $G$ has an architecture $A(\nlayers',\nvec', \activvec')$, 
possibly wider and deeper than the target network $\Fnn$.
The \emph{pruned network} $\hat{G}$ is obtained by pruning (setting to 0) many weights
of the large network $G$.
For each layer $i\in[\nlayers']$, and each pair of neurons $j_1\in[n_i]$ and $j_2\in[n_{i-1}]$,
for the weight $w_{i, j_1, j_2}$ of the large network $G$
there is a corresponding mask $\mask_{i, j_1, j_2}\in\{0, 1\}$ such that
the weight of the pruned network $\hat{G}$ is $w'_{i, j_1, j_2} = \mask_{i, j_1, j_2}w_{i, j_1, j_2}$.
\warning{$w'$ never used anymore?}
The pruned network $\hat{G}$ will have a different architecture from $\Fnn$,
but at a higher level (by grouping some neurons together) it will have the same `virtual' architecture,
with virtual weights $\hat{W}$.
As in previous theoretical work, we consider an `oracle' pruning procedure, as our objective is to understand the limitations of even the best pruning procedures.

For a matrix $M\in[-c, c]^{n\times m}$,
we denote by % $\|M\|_2 = \sigma_{\max}(M)$ 
$\|M\|_2$ its spectral norm, equal to its largest singular value,
and its max-norm is $\|M\|_{\max} = \max_{i,j} |M_{i, j}|$.
%The entry-wise matrix $2$-norm (the Frobenius norm) and $\max$-norm are defined to be the norms of a matrix linearized to a vector:
%\begin{align*}
%    \|M\|_2
%    = \biggl(\sum_{i, j} |M_{i, j}|^2\biggr)^{\frac{1}{2}}
%    \ , \quad
%    \|M\|_{\max} = \max_{i,j} |M_{i, j}|.
%\end{align*}
%The values of interest for $p$ are 1, 2 and $\max$.
%The matrix $2$-norm coincides with the Frobenius norm, i.e. $\|M\|_2 = \sqrt{\tr(M^\top M)}$, and
In particular, for a vector $v$,
we have $\|Mv\|_2\leq \|M\|_2\|v\|_2$ and 
$\|M\|_{\max} \leq \|M\|_2 \leq{\sqrt{nm}\|M\|_{\max}}$
and also $\|M\|_{\max} \leq c$.
This means for example that $\|M\|_2 \leq 1$ % for some $M$ 
is a stronger condition than $\|M\|_{\max}\leq 1$.

\iffalse
We denote by $\|M\|$ the spectral norm of the matrix $M$
and the entry-wise max-norm is $\|M\|_{\max} = \max_{i,j} |M_{i, j}|$.
%The entry-wise matrix $2$-norm (the Frobenius norm) and $\max$-norm are defined to be the norms of a matrix linearized to a vector:
%\begin{align*}
%    \|M\|_2
%    = \biggl(\sum_{i, j} |M_{i, j}|^2\biggr)^{\frac{1}{2}}
%    \ , \quad
%    \|M\|_{\max} = \max_{i,j} |M_{i, j}|.
%\end{align*}
%The values of interest for $p$ are 1, 2 and $\max$.
%The matrix $2$-norm coincides with the Frobenius norm, i.e. $\|M\|_2 = \sqrt{\tr(M^\top M)}$, and
In particular,
$\|Mv\|_2 \leq \|M\|\|v\|_2$ for a matrix $M$ and a vector $v$,
and for $M\in[-c, c]^{n\times m}$ we have
$\|M\|_{\max} \leq \|M\| \leq{\sqrt{nm}\|M\|_{\max}}$
and $\|M\|_{\max} \leq c$.
This means for example that $\|M\| \leq 1$ % for some $M$ 
is a stronger condition than $\|M\|_{\max}\leq 1$.
\fi

%Notice that one may write 
%$\|M\|_p = \left\|[ \|M_{i,.}\|_p]^\top_i\right\|_p = \left\|[ \|M_{.,j}\|_p]^\top_j\right\|_p$.
%\comment{O: his is actually true for any partitioning of a big vector (the linearized matrix)}

%%%%%%%%%%%%%%%%%%%%%%%%%%%%%%%%%%%%%%%%%%%%%%%%%%%%%%%%%%%%%%%
\section{Objective}
\label{sec:obj}
%%%%%%%%%%%%%%%%%%%%%%%%%%%%%%%%%%%%%%%%%%%%%%%%%%%%%%%%%%%%%%%

\textbf{Objective:} Given an architecture $\archi$ and accuracy $\epsilon > 0$,
construct a network $G$ from some larger architecture $A(\nlayers',\nvec', \activvec')$, 
% \comment{L: I worry that this will lead to unreadable notation on the indices, like $n_{2i+1}$ which will lead to headaches} \comment{O: yes I get the point, good to keep thinks readable here} 
such that if the weights of $G$ are randomly initialized  \emph{(no training)}, then
for any target network $\Fnn$ from $\archi$,
setting some of the weights of $G$ to 0 \emph{(pruning)} reveals a subnetwork $\hat{G}$ such that with high probability,
\begin{align*}
\sup_{x\in\inputset} \|\Fnn(x) - \hat{G}(x)\|_2 \leq \eps
\end{align*}
\textbf{Question:} How large must $G$ be to contain all such $\hat{G}$? More precisely, how many more neurons or how many more weights must $G$ have compared to $\Fnn$?

\citet{ramanujan2019s} were the first to provide a formal asymptotic argument proving that such a $G$ can indeed exist at all.
\citet{malach2020proving} went substantially further by providing the first polynomial bound on the size of $G$ compared to the size of the target network $\Fnn$.
To do so, they make the following assumptions on the target network: (i) the inputs $x\in\inputset$ must satisfy $\|x\|_2 \leq 1$, 
and at all layers $i\in[\nlayers]$: 
(ii) the weights must be bounded in $[-1/\sqrt{\nmax}, 1/\sqrt{\nmax}]$, %: $\|\WF_i\|_{\max}\leq \sqrt{\nmax}$
(iii) they must satisfy $\|\WF_i\|_2 \leq 1$ at all layers $i$, and
(iv) the number of non-zero weights at layer $i$ must be less than $\nmax$: $\|\WF_i\|_0 \leq \nmax$.
Note that these constraints imply that $\Fmax(\inputset)\leq 1$.
Then under these conditions,
they prove that any ReLU network with $\nlayers$ layers and $\nmax$ neurons per layer
can be approximated%
\footnote{Note that even though their bounds are stated in the 1-norm, this is because they consider a single output---for multiple outputs their result holds in the 2-norm, which is what their proof uses.}
within $\eps$ accuracy with probability $1-\delta$
by pruning a network $G$ with $2\nlayers$ ReLU layers and each added intermediate layer has $\nmax^2\lceil \frac{64\nlayers^2\nmax^3}{\eps^2}\log\frac{2\nmax^2\nlayers}{\delta}\rceil$
neurons.
These assumptions are rather strong, as in general this forces the activation signal to decrease quickly
with the depth.
Relaxing these assumptions while using the same proof steps would make the bounds exponential in the number of layers.
% Our work builds upon these first theoretical results and re-uses some of their techniques to provide sharper bounds while removing these assumptions.
We build upon the work of \citet{ramanujan2019s,malach2020proving},
who gave the first theoretical results on the Lottery Ticket Hypothesis, albeit under restrictive assumptions.
Our work re-uses some of their techniques to provide sharper bounds while removing these assumptions.

% \comment{Malach et al. paper: \url{https://arxiv.org/pdf/2002.00585.pdf}}

%\todo{?Say that this paper is not concerned about \emph{how} to prune from input/output examples,
%but rather what conditions on the architecture are sufficient to allow for a pruning procedure
%to obtain a good pruned network at all.}

%%%%%%%%%%%%%%%%%%%%%%%%%%%%%%%%%%%%%%%%%%%%%%%%%%%%%%%%%%%%%%%
\section{Approximation Propagation}\label{sec:approx_propag}
%%%%%%%%%%%%%%%%%%%%%%%%%%%%%%%%%%%%%%%%%%%%%%%%%%%%%%%%%%%%%%%

In this section, we analyze how the approximation error
between two networks with the same architecture
propagates through the layers.
The following lemma is a generalization of the (end of the) proof of \citet[Theorem~A.6]{malach2020proving}
that removes their aforementioned assumptions 
and provides better insight into the impact of the different variables on the required accuracy, but is not sufficient in itself to obtain better bounds.
For two given networks with the same architecture,
it determines what accuracy is needed on each individual weight
so the outputs of the two neural networks differ by at most $\eps$ on any input.
Note that no randomization appears at this stage.

%% Laurent: See laptop: /Prog/Racket/neurips2020_nn/misc.rkt
\begin{lemma}[Approximation propagation]
\label{thm:propagation}
Consider two networks $\Fnn$ and $\hat{G}$ with the same architecture $\archi$
with respective weight matrices $\WF$ and $\hat{W}$,
each weight being in $[-\wmax, \wmax]$.
Given $\eps>0$, 
if for each weight $\wF$ of $\Fnn$ the corresponding weight $\hat{w}$ of $\hat{G}$ 
we have $|\wF - \hat{w}| \leq \eps_w$,
and if 
\begin{align*}
    \eps_w \leq \eps\left/\left(e~\nlayers~\lambdamax~\nmax^{\nicefrac32}~\Fmax(\inputset)~
    \prod_{i=1}^\nlayers \max\{1, \lambda_i\|\hat{W}_{i}\|_2\}\right)\right.,
%\end{align*}
%then 
%\begin{align*}
    \quad\text{then}\quad\sup_{x\in\inputset}\|\Fnn(x) - \hat{G}(x)\|_2 \leq \eps\,.
\end{align*}
\end{lemma}
\newcommand{\propagproof}{
\begin{proof}
For all $x\in\inputset$, for a layer $i$:
\begin{align*}\resetalph{}
    \|\Fnn_i(x) - \hat{G}_i(x)\|_2 
    &\alphoverset{=} \|\fnn_i(\Fnn_{i-1}(x)) - \hat{g}_i(\hat{G}_{i-1}(x))\|_2 \\
    &=
    \left[ \sum_{j\in [n_i]} (\activ_i({\WF_{i,j}} \Fnn_{i-1}(x)) 
            -\activ_i({\hat{W}_{i, j}} \hat{G}_{i-1}(x)))^2
    \right]^{1/2} \\
    &\alphoverset{\leq}
    \lambda_i
    \|{\WF_{i}} \Fnn_{i-1}(x) - {\hat{W}_{i}} \hat{G}_{i-1}(x)\|_2\\
    &\alphoverset{\leq}
    \lambda_i
    \|{\hat{W}_{i}} (\Fnn_{i-1}(x) - \hat{G}_{i-1}(x))\|_2 +
    \lambda_i
    \|({\WF_{i}} - {\hat{W}_{i}}) \Fnn_{i-1}(x)\|_2\\
    &\alphoverset{\leq}
    \lambda_i\|\hat{W}_{i}\|_2\,\|\Fnn_{i-1}(x) - \hat{G}_{i-1}(x)\|_2 + 
    \lambda_i\|{\WF_{i}} - {\hat{W}_{i}}\|_2\, \|\Fnn_{i-1}(x)\|_2 \\
    &\alphoverset{\leq}
    \lambda_i\|\hat{W}_{i}\|_2\|\Fnn_{i-1}(x) - \hat{G}_{i-1}(x)\|_2 + 
    \lambdamax\sqrt{n_in_{i-1}}\eps_w\wmax\sqrt{n_i}\Fmax(x) \\
    &\leq \lambda_i \|\hat{W}_{i}\|_2\|\Fnn_{i-1}(x) - \hat{G}_{i-1}(x)\|_2 + 
    \eps_w \lambdamax\nmax^{\nicefrac32}\Fmax(x) \\
    &\alphoverset{\leq}
    \eps_w ei\lambdamax \nmax^{\nicefrac32}\Fmax(x)
    \prod_{u=1}^i \max\{1, \lambda_u\|\hat{W}_{u}\|_2\}
    %2^{-m+1}\sqrt{\nmax}\Fmax(x)\wmax^i\left(\prod_{u=1}^i \lambda_u \sqrt{n_un_{u-1}}\right)
\end{align*}\resetalph{}%
where
\alphnextref{} follows from the definition of $\Fnn_i$ and $\hat{G}_i$,
\alphnextref{} follows from $|\activ_i(x) - \activ_i(y)| \leq \lambda_i|x-y|$ by the definition of $\lambda_i$,
%for some adequate $\lambda_i\geq0$ that depends on $\activ_i$ from \cref{eq:bounded_diff},
\alphnextref{} follows from the Minkowski inequality,
\alphnextref{} follows from $\|Mv\|_2 \leq \|M\|_2\|v\|_2$
%\cref{lem:submult}
applied to both terms,
\alphnextref{} is by assumption that $|\wF - \hat{w}|\leq \eps_w$
and $\|M\|_2 \leq c\sqrt{ab}$ for any $M\in[-c, c]^{a\times b}$,
and finally \alphnextref{} follows from \cref{lem:ax+b_protected}, using $\|\Fnn_0(x) - \hat{G}_0(x)\|_2 = 0$.
Therefore
\begin{align*}
    \|\Fnn(x) - \hat{G}(x)\|_2
    = \|\Fnn_\nlayers(x) - \hat{G}_\nlayers(x)\|_2 
    \leq 
    \eps_w e\nlayers\lambdamax \nmax^{\nicefrac32}\Fmax(x)
     \prod_{i=1}^\nlayers \max\left\{1,\lambda_i\|\hat{W}_{i}\|_2\right\}
\end{align*}
and taking $\eps_w$ as in the theorem statement proves the result.
\end{proof}
}
The proof is given in \cref{sec:technical}.
%%%%%%%%%%%%%%%%%%%%%%%%%
%\propagproof   %%% See appendix
%%%%%%%%%%%%%%%%%%%%%%%%%
\begin{example}\label{ex:propag}
Consider an architecture with only ReLU activation function ($\lambda=1$),
weights in $[-1, 1]$
and assume that $\Fmax(\inputset) = 1$
and take the worst case $\|\hat{W}_i\|_2\leq\wmax \nmax =\nmax$,
then \cref{thm:propagation} tells us that 
the approximation error on each individual weight must be at most
    $\eps_w \leq \eps/(e\nlayers \nmax^{\nicefrac32+\nlayers})$
%\begin{align*}
%    \eps_w \leq \eps\left/\left(e\nlayers \nmax^{\nicefrac32+\nlayers}\right) \right.
%\end{align*}
so as to guarantee that the approximation error between the two networks is at most $\eps$.
%\comment{Actually the factor $\nlayers$ goes away in this case by using \cref{lem:ax+b} instead of \cref{lem:ax+b_protected}}
This is exponential in the number of layers.
If we assume instead that $\|\hat{W}_i\|_2\leq1$ as in previous work then this reduces to a mild polynomial dependency:
    $\eps_w \leq \eps/(e\nlayers \nmax^{\nicefrac32})$.
%\begin{align*}
%    \eps_w \leq \eps\left/\left(e\nlayers \nmax^{\nicefrac32}\right)\right.\,.
%    &\qedhere
%\end{align*}
\end{example}

%%%%%%%%%%%%%%%%%%%%%%%%%%%%%%%%%%%%%%%%%%%%%%%%%%%%%%%%%%%%%%%
\section{Construction of the ReLU Network $G$ and Main Ideas}\label{sec:constructG}
%%%%%%%%%%%%%%%%%%%%%%%%%%%%%%%%%%%%%%%%%%%%%%%%%%%%%%%%%%%%%%%

% TODO: Add this to the table of notation
\newcommand{\pwplus}{p_{w\scriptscriptstyle\geq0}}%{p^{\scriptscriptstyle+}_w} % 1/w distribution
\newcommand{\pw}{p_w} % ±1/v
\newcommand{\pwprod}{p_{w\scriptscriptstyle\times}} % product distribution
\newcommand{\pab}{p^{\scriptscriptstyle+}}%{p^{(12)}} % product dist of \wa and \wb
\newcommand{\pcd}{p^{\scriptscriptstyle-}}%{p^{(34)}} % product dist of \wc and \wd
\newcommand{\zpin}{z^{\scriptscriptstyle+}_{\text{in}}}
\newcommand{\zpout}{z^{\scriptscriptstyle+}_{\text{out}}}
\newcommand{\zmin}{z^{\scriptscriptstyle-}_{\text{in}}}
\newcommand{\zmout}{z^{\scriptscriptstyle-}_{\text{out}}}
\newcommand{\wa}{{\zpout}{}}%{w^{(1)}}
\newcommand{\wb}{{\zpin}{}}%{w^{(2)}}
\newcommand{\wc}{{\zmout}{}}%{w^{(3)}}
\newcommand{\wD}{{\zmin}{}}%{w^{(4)}} % \wd already defined :(

We now explain how to construct the large network $G$ given only the
architecture $\archi$, the accuracy $\eps$, and the domain $[-\wmax, \wmax]$ of the weights.
Apart from this, the target network $\Fnn$ is unknown.
In this section all activation functions are ReLU $\activ(x)=\max\{0, x\}$,
and thus $\lambda=1$.

We use a similar construction of the large network $G$ as \citet{malach2020proving}:
both the target network $\Fnn$ and the large network $G$ consist of fully connected ReLU layers,
but $G$ may be wider and deeper.
The weights of $\Fnn$ are in $[-\wmax, \wmax]$.
The weights for $G$ (at all layers) are all sampled from the same distribution, the only difference with the previous work is the distribution of the weights: we use a hyperbolic distribution instead of a uniform one.

Between layer $i-1$ and $i$ of the target architecture, 
for the large network $G$
we insert an intermediate layer $i-\nicefrac12$ of ReLU neurons.
Layer $i-1$ is fully connected to layer $i-\nicefrac12$ which is fully connected to layer $i$.
By contrast to the target network $\Fnn$, in $G$ the layers $i-1$ and $i$ are not directly connected.
The insight of \citet{malach2020proving} is to use two intermediate (fully connected ReLU) neurons $z^+$ and $z^-$ of the large network $G$ to mimic one weight $\wF$ of the target network
(see\cref{fig:weight_decomp}):
Calling $\zpin, \zpout,\zmin, \zmout$ the input and output weights of $z^+$ and $z^-$
that match the input and output of the connection $\wF$,
then in the pruned network $\hat{G}$ all connections apart from these 4 are masked out by a binary mask $\vmask$ set to 0.
These two neurons together implement a `virtual' weight $\hat{w}$ and calculate the function $x\mapsto \hat{w}x$ by taking advantage of the identity $x = \activ(x) - \activ(-x)$: % (see \cref{fig:wx}):
\begin{align*}
    \hat{w} = \zpout\activ(\zpin x) + \zmout\activ(\zmin x)
\end{align*}
Hence, if $\zpin \approx \wF \approx - \zmin$ and $\zpout \approx 1 \approx -\zmout$,
the virtual weight $\hat{w}$ made of $z^+$ and $z^-$ is approximately $\wF$.
Then, for each target weight $\wF$, \citet{malach2020proving} sample many such intermediate neurons to ensure that two of them can be pruned so that $|\wF - \hat{w}|\leq \eps_w$ with high probability.
This requires $\Omega(1/\eps_w^2)$ samples and, when combined with \cref{thm:propagation} (see \cref{ex:propag}),
makes the general bound on the whole network grow exponentially in the number of layers,
unless strong constraints are imposed.

\iffalse
\begin{figure}
    \centering
    \includegraphics[width=.3\textwidth]{img/weight-decomp.png}
    \caption{Mimicking the function $x\mapsto wx$ by adding two ReLU neurons,
    $x\mapsto \activ(wx)$ and $x\mapsto - \activ(-wx)$.    }
    \label{fig:wx}
\end{figure}
\fi

%%% Drawing:
%%% https://docs.google.com/drawings/d/15jC_VylsgQIyhxHFaWndPFEY0mVDsxNI61qhhxfGDec/edit
\begin{figure}
    \centering
    \includegraphics[width=0.85\textwidth]{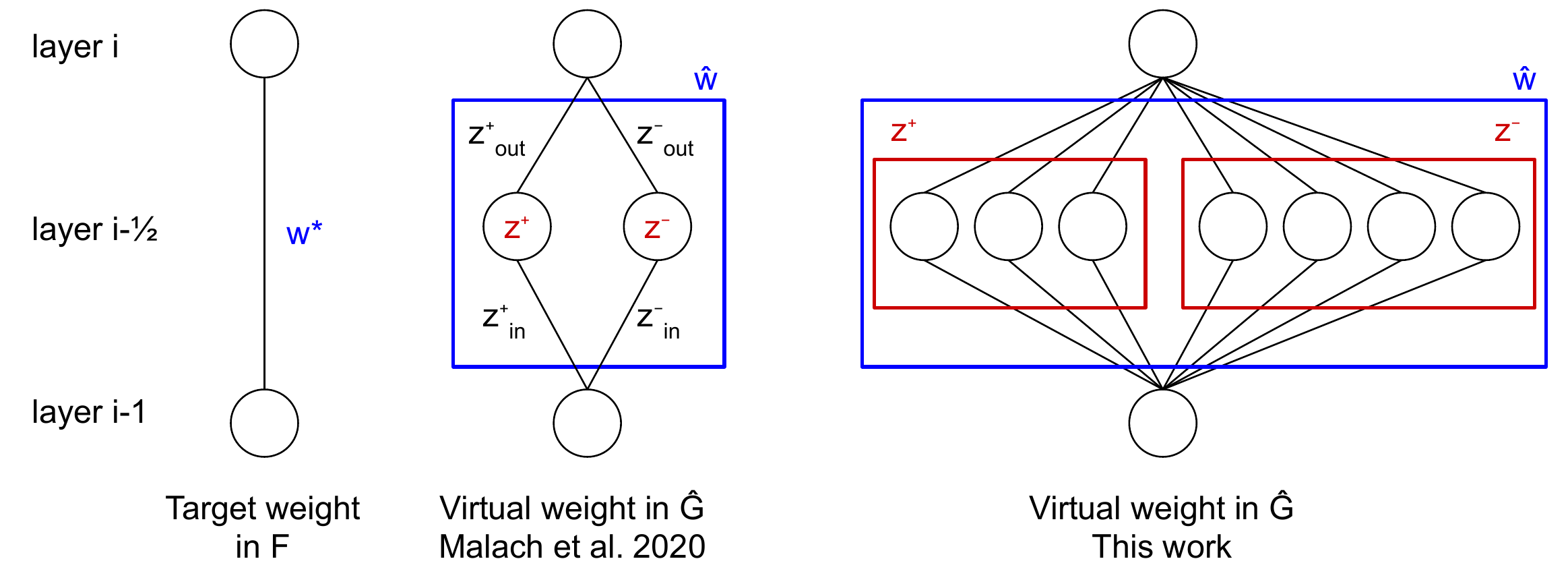}
    \caption{%\todo{Replace by figure for unpruned NN as discussed.}
    The target weight $\wF$ is simulated
    in the \emph{pruned} network $\hat{G}$
    by 2 intermediate neurons, requiring $1/\eps^2$ sampled neurons (previous work)
    or by $2\log 1/\eps$ intermediate neurons due to a `binary' decomposition of $\wF$,
    requiring only $O(\log 1/\eps)$ sampled neurons (this work).}
    \label{fig:weight_decomp}
\end{figure}

To obtain a logarithmic dependency on $\eps_w$, we use three new insights that take advantage of the composability of neural networks: 1) `binary' decomposition of the weights, 2) product weights, and 3) batch sampling. We detail them next.

%\subsubsection*{Golden-ratio decomposition}
\paragraph{Weight decomposition.}
Our most important improvement is to build the weight $\hat{w}$ not with just two intermediate neurons, but with $O(\log\nicefrac{1}{\eps})$ of them, so as to decompose the weight into pieces
of different precisions, and recombine them with the sum in the neuron at layer $i+1$ (see \cref{fig:weight_decomp}), using a suitable binary mask vector $\vmask$ in the pruned network $\hat{G}$.
Intuitively, the weight $\hat{w}$ is decomposed into its binary representation up to a precision of $k\approx \lceil\log_2 1/\eps\rceil$ bits: $\sum_{s=1}^k \mask_s 2^{-s}$. 
Using a uniform distribution to obtain these weights $2^{-s}$ would not help, however.
But, because the high precision bits are now all centered around 0,
we can use a hyperbolic sampling distribution $\pw(|w|) \propto 1/w$ which has high density near 0.
More precisely, but still a little simplified,
for a weight $\wF\in[-1, 1]$ we approximate $\wF$ within $\approx 2^{-k}$ accuracy
with the virtual weight $\hat{w}$ such that:
\begin{align}\label{eq:bindecomp}
    \hat{w}x =
    \sum_{s=1}^k \mask_s \left[\wa_{, s} \activ(\wb_{, s} x) + \wc_{, s}\activ(\wD_{, s} x) \right]
    \approx \sum_{s=1}^k \mask_s \sgn(\wF)2^{-s}x \approx \wF x
\end{align}
where $\mask_s\in\{0, 1\}$ is factored out since all connections have the same mask,
and where $\wa_{, s}\wb_{, s}\approx \sgn(\wF)2^{-s}\approx \wc_{, s}\wD_{, s}$
and $\wa_{, s}>0$, $\sgn(\wb_{, s})=\sgn(\wF)$, $\wc_{, s}<0$ and $\wD_{, s} = -\sgn(\wF)$.
Note however that, because of the inexactness of the sampling process,
we use a decomposition in base $\nicefrac32$ instead of base 2 
(\cref{lem:grd} in \cref{sec:random_continuous}).

\paragraph{Product weights.}
Recall that $\wa_{, s} \activ(\wb_{, s} x)=\wa_{, s} \max\{0, \wb_{, s} x\}$.
For fixed signs of $\wa_{, s}$ and $\wb_{, s}$, this function can be equivalently calculated
for all possible values of these two weights such that the product $\wa_{, s}\wb_{, s}$ remains unchanged.
Hence, forcing $\wa_{, s}$ and $\wb_{, s}$ to take 2 specific values is wasteful
as one can take advantage of the cumulative probability mass of all their combinations.
We thus make use of the induced product distribution,
which avoids squaring the number of required samples.
Define the distribution $\pwplus$ for positive weights $w\in[α, β]$ with $0< α <β$ and
$\pw$, symmetric around 0, for $w\in[-β, -α]\cup[α, β]$:
\begin{align*}
    \pwplus(w) = \frac{1}{w\ln(\beta/\alpha)}\propto\frac{1}{w}\,, \quad\text{and}\quad
     \pw(w) = \pw(-w) = \tfrac12\pwplus(|w|) = \frac{1}{2|w|\ln(β/α)}\,.
\end{align*}
Then, instead of sampling uniformly until both $\wa_{,s}\!\approx\!1$ and 
$\wb_{,s}\!\approx\!\wF$, we sample both from $\pw$ so that $\wa_{, s}\wb_{, s} \approx \wF$,
taking advantage of the induced product distribution $\pwprod\!\approx\!\tfrac12 \pwplus$
(\cref{lem:prodw}).

\paragraph{Batch sampling.} 
Sampling sufficiently many intermediate neurons so that a subset of them are employed in approximating one target weight $\wF$ with high probably and then discarding (pruning) all other intermediate neurons is wasteful.
Instead, we allow these samples to be `recycled' to be used for other neurons in the same layer.
This is done by partitioning the neurons in different buckets (categories) and ensuring
that each bucket has enough neurons (\cref{lem:fillcat}).

%%%%%%%%%%%%%%%%%%%%%%%%%%%%%%%%%%%%%%%%%%%%%%%%%%%%%%%%%%%%%%%
\section{Theoretical Results}\label{sec:relu}
%%%%%%%%%%%%%%%%%%%%%%%%%%%%%%%%%%%%%%%%%%%%%%%%%%%%%%%%%%%%%%%

We now have all the elements to present our central theorem,
which tells us how many intermediate neurons to sample to 
approximate all weights at a layer of the target network with high probability.
\cref{rem:nperw} below will then describe the result in terms of number of neurons
per target weight.

\begin{theorem}[ReLU sampling bound]\label{thm:relu}
For a given architecture $\archi$ where $\activ$ is the ReLU function,
%with $\NF$ weights ($N\leq (\nlayers+1)\nmax^2$),
with weights in $[-\wmax, \wmax]$
and a given accuracy $\eps$,
the network $G$ constructed as above 
with weights sampled from $\pw$ with $[α, β] = [α'/q, β'/q], α'=2\eps_w/9, β'=2\wmax/3$, and $q=(α'β')^{\nicefrac14}$,
requires only 
to sample $\Msim_i$ intermediate neurons for each layer $i$, where
\begin{align*}
    %\msim&=8\ceiling{\ln\frac{3\wmax}{ε_w}\log_{\nicefrac32}\frac{2kN}{δ}}
    \Msim_i &= \ceiling{16k'\left(n_in_{i-1} + \ln \frac{2\nlayers k'}{\delta}\right)}
    \quad\text{ with }\quad k'=\log_{\nicefrac{3}{2}}\frac{3\wmax}{ε_w}
    \quad\text{ and}\\
    \eps_w &= \eps\left/\left(e~\nlayers~\nmax^{\nicefrac32}~\Fmax(\inputset)~
    \prod_{i=1}^\nlayers \max\{1,\|\hat{W}_{i}\|_2\}\right)\right.
\end{align*}
($\eps_w$ is in \cref{thm:propagation} with $\lambda=1$ for ReLU),
in order to ensure that for any target network $\Fnn$ with the given architecture $\archi$,
there exist binary masks $\vmask_{i, j} = (b_{i, j, 1},\ldots b_{i, j, n_{i-1}})$ of $G$ such that for the resulting subnetwork $\hat{G}$,
\begin{align*}
    \sup_{x\in\inputset} \|\Fnn(x)-\hat{G}(x)\|_2 \leq \eps\,.
\end{align*}
\end{theorem}
\begin{proof}

\textbf{Step 1. Sampling intermediate neurons to obtain product weights.}
Consider a single target weight $\wF$.
Recalling that $\wa_{, s}>0$ and $\wc_{, s} <0$,
we rewrite \cref{eq:bindecomp} as 
\begin{align*}
    \hat{w}x &= \sum_{s=1}^k \mask_s \wa_{, s}\activ(\wb_{, s}x) + \sum_{s=1}^k \mask_s \wc_{, s}\activ(\wD_{, s}x) \\
    %&= \sum_{s=1}^k \mask_s \sgn(\wa_{, s})\activ(|\wa_{, s}|\wb_{, s}x) + \sum_{s=1}^k \mask_s \sgn(\wc_{, s})\activ(|\wc_{, s}|\wD_{, s}x) \\
    &= \sum_{s=1}^k \mask_s \activ(\underbrace{\wa_{, s}\wb_{, s}}_{\hat{w}^+}x) + \sum_{s=1}^k -\mask_s \activ(-\underbrace{\wc_{, s}\wD_{, s}}_{\hat{w}^-}x)
    %&= 
    %\underbrace{\sum_{s=1}^k \mask_s \wa_{, s}|\wb_{, s}|\activ(\sgn(\wb_{, s})x)}_{z^+(x) = \hat{w}^+ \activ(\sgn(\wF)x)} + \underbrace{\sum_{s=1}^k \mask_s \wc_{, s}|\wD_{, s}|\activ(\sgn(\wD_{, s})x)}_{z^-(x) = -\hat{w}^- \activ(-\sgn(\wF)x)}\,.
\end{align*}
The two virtual weight $\hat{w}^+$ and $\hat{w}^-$ are obtained separately.
We need both $|\wF- \hat{w}^+| \leq \eps_w/2$ and $|\wF- \hat{w}^-| \leq \eps_w/2$
so that $|\wF - \hat{w}|\leq \eps_w$.

Consider $\hat{w}^+$ (the case $\hat{w}^-$ is similar).
We  now sample $\msim$ intermediate neurons,
fully connected to the previous and next layers,
but only keeping the connection between the same input and output neurons as $\wF$
(the other weights are zeroed out by the mask $\vmask$).
%Note that many sampled neurons may end up having all their connections masked to 0.
%\comment{We could take advantage of the full connection: if the current sample (product) weight
%is not good enough, maybe it's good enough for another weight? but must avoid sharing}
For a single sampled intermediate neuron $z$, 
all its weights, including $\wb$ and $\wa$, are sampled from $\pw$,
thus the product $|\wa\wb|$ is sampled from the induced product distribution $\pwprod$
and, a quarter of the time, $\wa$ and $\wb$ have the correct signs (recall we need $\wa>0$ and $\sgn(\wb)=\sgn(\wF)$).
Define
\begin{align*}
    \pab(\wa\wb) = P(w=\wa\wb &~\land~ \wa\sim\pw ~\land~ \wa>0 \\
    &~\land~ \wb\sim\pw ~\land~ \sgn(\wb)=\sgn(\wF))
\end{align*}
then with
$%\begin{align*}
    \pab(\wa\wb) ~\geq~ \pwprod(|\wa\wb|)/4 ~\geq~ \pwplus(|\wa\wb|)/8
$ %\end{align*}
where the last inequality follows from \cref{lem:prodw} for $|\wa\wb|\in[α', β']$,
$\wa\in[α, β]$ and $\wb\in[α, β]$,
and similarly for $\wc\wD$ with $\pcd(\wc\wD)\geq \pwplus(|\wc\wD|)/8$.

Note that because $\sgn(\wa) = - \sgn(\wc)$ and $\sgn(\wb) = -\sgn(\wD)$,
the samples for $\hat{w}^+$ and $\hat{w}^-$ are mutually exclusive 
which will save us a factor 2 later.

%After this stage, we do not need to mention the intermediate \emph{neurons} anymore,
%as all that matter are the product weights $\wa\wb$ and $\wc\wD$.

\textbf{Step 2. `Binary' decomposition/recomposition.}
Consider a target weight $\wF \in [-\wmax, \wmax]$.
Noting that \cref{cor:samplemax} equally applies for negative weights by first negating them,
we obtain $\hat{w}^+$ and $\hat{w}^-$
by two separate applications of \cref{cor:samplemax}
where we substitute
$P_\eps\leadsto P_\eps/8=\pwplus/8$,
$\eps\leadsto\eps_w/2$, $\delta\leadsto\delta_w$.
%%% Comment: We don't need \delta\leadsto\delta_w/2 because this is done in Step 3 with all the rest.
%
Substituting $P_\eps$ with $\pwplus/8$ in \cref{eq:cnorm} shows that this leads to a factor 8 on $\msim$.
Therefore,
by sampling $\msim=8\lceil k'\ln\frac{k'}{δ_w}\rceil$ weights
  from $\pwprod$
  in $[α', β'] = [2ε_w/9, 2\wmax/3]$
with $k'=\log_{\nicefrac{3}{2}}\frac{3\wmax}{ε_w}$
ensures that there exists a binary mask $\vmask$ of size at most $k'$ such that
$|\wF - \hat{w}^+| \leq \eps_w/2$ with probability at least $1-\delta_w$.
We proceed similarly for $w^-$.
Note that \cref{cor:samplemax} guarantees $|\hat{w}|\leq |\wF| \leq \wmax$,
even though the large network $G$ may have individual weights larger than $\wmax$.

\textbf{Step 2'. Batch sampling.}
Take $k := \lceil\log_{\nicefrac32}\frac{\wmax}{2\eps_w}\rceil \leq k'$
to be the number of `bits' required to decompose a weight with \cref{cor:samplemax} (via \cref{lem:grd}).
Sampling $\msim$ different intermediate neurons for each target weight and discarding $\msim-k$ samples is wasteful: Since there are $n_i n_{i-1}$ target weights at layer $i$, we would need $n_in_{i-1}\msim$ intermediate neurons, when in fact most of the discarded neurons could be recycled for other target weights.

Instead, we sample all the weights of layer $i$ at the same time,
requiring 
that we have at least $n_i n_{i-1}$ samples for each of the $k$ intervals
of the `binary' decompositions of $\hat{w}^+$ and $\hat{w}^-$.
Then we use \cref{lem:fillcat} with $2k$ categories:
The first $k$ categories are for the decomposition of $\hat{w}^+$ and the next $k$ ones
are for $\hat{w}^-$.
Note that these categories are indeed mutually exclusive as explained in Step 1. and,
adapting \cref{eq:cnorm},
each has probability at least $\frac18\int_{w=\gamma^{u+1}}^{\gamma^u}\pwplus(w)dw\geq 1/(8\log_{\nicefrac32} (3\wmax/\eps_w)) = 1/(8k')$ (for any $u$).
Hence, using \cref{lem:fillcat} where we take $n\leadsto n_in_{i-1}$
and $\delta\leadsto\delta_i$,
we only need to sample
$\lceil 16k'(n_in_{i-1} + \ln\frac{2k}{\delta_i})\rceil \leq \lceil 16k'(n_in_{i-1} + \ln\frac{2k'}{\delta_i})\rceil = M_i$
%\begin{align*}
%    \ceiling{16k'\left(n_in_{i-1} + \ln\frac{2k}{\delta_i}\right)} \leq \ceiling{16k'\left(n_in_{i-1} + \ln\frac{2k'}{\delta_i}\right)} = M_i
%\end{align*}
intermediate neurons to ensure that with probability at least $1-\delta_i$
each $\hat{w}^+$ and $\hat{w}^-$ can be decomposed into $k$ product weights in each 
of the intervals of \cref{lem:grd}.

%\textbf{Step 3. Global high probability.}
%Since each $\wF$ is associated with 2 weights, $\hat{w}^+$ and $\hat{w}^-$, 
%for the accuracy $|\wF-\hat{w}|\leq \eps_w$ to hold simultaneously for all $\NF$ target weights
%with probability at least $1-\delta$,
%using a union bound 
%we need $\delta \leq 2N\delta_w$.

\textbf{Step 3. Network approximation.}
Using a union bound, we need $\delta_i = \delta/\nlayers$
for the claim to hold simultaneously for all $\nlayers$ layers.
Finally, when considering only the virtual weights $\hat{w}$ (constructed from $\hat{w}^+$ and $\hat{w}^-$),
$\hat{G}$ and $\Fnn$ now have the same architecture, hence choosing $\eps_w$ as in \cref{thm:propagation} 
ensures that with probability at least $1-\delta$, $\sup_{x\in\inputset}\|\Fnn(x)-\hat{G}(x)\| \leq \eps$.
\end{proof}
%%%%%%%%%%%

\begin{remark}\label{rem:nperw}
Consider $n_i=\nmax$ for all $i$ and assume $\|W_i\|_2 \leq 1$, $\wmax = 1$
and $\Fmax(\inputset) \leq 1$.
Then $\eps_w \geq \eps/(e\nlayers \nmax^{\nicefrac32})$
and $k'\leq\log_{\nicefrac32} (3e\nlayers\nmax^{\nicefrac32}/\eps)$.
Then we can interpret \cref{thm:relu} as follows:
When sampling the weights of a ReLU architecture from the hyperbolic distribution,
we only need to sample $\Msim_i/\nmax^2 \leq 16k' + \ln(2\nlayers k'/\delta)/\nmax^2 =$ {\boldmath $\tilde{O}(\log(\nlayers\nmax/\eps))$} neurons per target weight (assuming $\nmax^2 > \log (\nlayers k'/\delta)$).
Compare with the bound of \citet[Theorem A.6]{malach2020proving} which,
under the further constraints that $\wmax \leq 1/\sqrt{\nmax}$
and $\max_{i\in[\nlayers]}\|\WF_i\|_0 \leq \nmax$
and with uniform sampling in $[-1, 1]$,
needed to sample $\Msim_i/\nmax^2 = \lceil64\nlayers^2\nmax^3\log(2N/\delta)/\eps^2\rceil$
neurons per target weight.
\end{remark}

\begin{example}\label{ex:numeric}
Under the same assumptions as \cref{rem:nperw},
for $\nmax=100, \nlayers=10, \eps=0.01, \delta=0.01$,
the bound above for \citet{malach2020proving} gives $\Msim_i/\nmax^2\leq 2\cdot 10^{15}$,
while our bound in \cref{thm:relu} gives $\Msim_i/\nmax^2 \leq 630$.
\end{example}
\begin{example}\label{ex:numeric_free}
Under the same conditions as \cref{ex:numeric},
if we remove the assumption that $\|W_i\|_2 \leq 1$, then \cref{thm:relu} gives
$M_i/\nmax^2 =$ {\boldmath $\tilde{O}(\nlayers\log(\nmax/\eps))$}
and numerically 
$M_i/\nmax^2 \leq 2\,450$.
\end{example}

We can now state our final result.
%\todo{Number of neurons?}
\begin{corollary}[Weight count ratio]\label{cor:mainresult}
Under the same conditions as \cref{thm:relu},
Let $\NF$ be the number of weights in the fully connected architecture $\archi$
and $N_{G}$ the number of weights of the large network $G$,
then the weight count ratio is $N_{G}/\NF \leq 32\nmax k' + \tilde{O}(\log(k'/\delta))$.
\end{corollary}
\begin{proof}
We have $\NF=\sum_{i=1}^\ell n_{i-1}n_i$,
and the total number of weights in the network $G$ if layers are fully connected is at most $N_{G}=\sum_{i=1}^\ell(n_{i-1}+n_i)M_i$,
where $M_i=16k'n_{i-1}n_i+O(\log(k'/\delta))$.
Hence the weight count ratio is $N_{G}/\NF \leq 32\nmax k' + \tilde{O}(\log(k'/\delta))$.
\end{proof}

\begin{remark}
Since in the pruned network $\hat{G}$ each target weight requires $k'$ neurons,
the large network has at most a constant factor more neurons than the pruned network.
\end{remark}

%\begin{align*}
%    \sum_{u=0}^{n^2} \binom{\msim}{u}\alpha^u(1-\alpha)^{\msim-u}
%\end{align*}

%\begin{align*}
%    m \geq \log_{\frac{1}{γ}}(\frac{3}{2\eps})n^2 + %\left[\log_{\frac{1}{γ}}(\frac{3}{2\eps})\right]^2\frac12\ln \frac{k}{\delta}
%\end{align*}

%%%%%%%%%%%%%%%%%%%%%%%%%%%%%%%%%%%%%%%%%%%%%%%%%%%%%%%%%%%%%%%
\section{Technical lemma: Random weights}
\label{sec:random_continuous}
%%%%%%%%%%%%%%%%%%%%%%%%%%%%%%%%%%%%%%%%%%%%%%%%%%%%%%%%%%%%%%%

The following lemma shows that if $\msim$ weights are sampled from a hyperbolic distribution,
we can construct a `goldary' (as opposed to `binary') representation of the weight
with only $\tilde O(\ln\frac1{ε}\ln\frac1{δ})$ samples.
Because of the randomness of the process, we use a ``base'' $3/2$ instead of a base 2 for logarithms, so 
that the different `bits' have overlapping intervals.
As the proof clarifies, the optimal base is $1/γ=\frac12(\sqrt{5}+1)\dot=1.62$. The base $1/γ=\nicefrac32$ is convenient.
The number $\frac12(\sqrt{5}+1)$ is known as the `golden ratio' in the mathematical literature, which explains the name we use.

\begin{lemma}[\boldmath Golden-ratio decomposition] % Hey Omar, did you see my message on slack?
\label{lem:grd} % all theorems/lemmas/etc should have a title
  For any given $\eps>0$ and $1/φ≤γ<1$, where $φ:=\frac12(\sqrt{5}+1)$ is the golden ratio, define the probability density
  $P_{\eps}(v):=\frac{c'}{v}$ for $v \in [εγ^2,γ]$ 
  with normalization $c':=[\ln\frac1{γε}]^{-1}$.
  For any $\delta\in(0,1)$,
  %if $\msim=\ceiling{\log_{\frac32}\frac{k}{δ}/c'}=\tilde{Ω}(\ln\frac1{ε}⋅\ln\frac1{δ})$
  if $\msim=\lceil k'\ln\nicefrac{k'}{δ}\rceil=\tilde{Ω}(\lnε⋅\lnδ)$
  with $k':=\log_γ(γε)$,
  then with probability at least $1-δ$ over the random sampling of $\msim$ `weights' $v_s \sim P_{\eps}$ for $s=1,...,\msim$, the following holds:
  For every `target weight' $w∈[0,1]$,
  there exists a mask $\vmask∈\{0,1\}^\msim$ with $|\vmask|≤ k'$
  such that $\hat w:=\mask_1 v_1+...+\mask_\msim v_\msim$ is $ε$-close to $w$,
  indeed $w-ε≤\hat w≤w$.
\end{lemma}
\vspace{1mm}

\iffalse % $γ=2/3$ version
\begin{lemma}[\boldmath $ε$-approximate weights by sampling $\tilde O(\ln\frac1{ε})$ weights from $P_{\eps}(v)\propto 1/v$]
\label{lem:grd} % all theorems/lemmas/etc should have a title
  For any given $\eps>0$, define the distribution
  $P_{\eps}(v):=\frac{c'}{v}$ for $v \in [\frac49ε,\frac23]$ 
  with normalization $c':=1/\ln\frac13{2ε}$.
  For any $\delta\in(0,1)$,
  %if $\msim=\ceiling{\log_{\frac32}\frac{k}{δ}/c'}=\tilde{Ω}(\ln\frac1{ε}⋅\ln\frac1{δ})$
  if $\msim=\lceil k'\ln\frac{k'}{δ}\rceil=\tilde{Ω}(\ln\frac1{ε}⋅\ln\frac1{δ})$
  with $k':=\log_{\nicefrac{3}{2}}\frac3{2ε}$,
  then with probability at least $1-δ$ over the random sampling of $\msim$ `weights' $v_s \sim P_{\eps}$ ($s=1,...,\msim$) the following holds:
  For every `weight' $w∈[0,1]$,
  there exists a mask $\vmask∈\{0,1\}^\msim$ with $|\vmask|≤ k'$
  such that $\hat w:=\mask_1 v_1+...+\mask_\msim v_\msim$ is $ε$-close to $w$,
  indeed $w≤\hat w≤w+ε$.
\end{lemma}
\vspace{1mm}
\fi

%Sampling weights uniformly or from Gaussian (as done in \cite{malach2020proving}), or indeed from any bounded probability density
%requires $O(1/ε)$ samples to generate small weights of size $O(ε)$ with high probability.
%Our hyperbolic distribution is scale invariant and requires only $\tilde O(\ln\frac1{ε})$ samples.
%This suggests to experiment with random hyperbolic weight initialization. 

\begin{proof}
Let $k = \lceil\log_γε\rceil \leq 1+\log_γε=k'$.
First, consider a sequence $(v_i)_{i\in[k]}$ such each $v_i$ is in the interval 
$I_i:=(γ^{i+1},γ^i]$ for $i=1,...,k$.
%Thus, at this stage $(v_{i})_{i\in[k]}$ is non-random.
We construct an approximating $\hat w$ for any weight $w_0:=w∈[0,1]$ by successively subtracting
$v_i$ when possible. Formally 
\begin{align*}
  \text{for}(i=1,...,k) 
  ~\{\text{if}~ w_{i-1}≥γ^i ~\text{then}~ \{ w_i:=w_{i-1}-v_i;~\mask_i=1 \} % \lceil 
                         ~\text{else}~ \{ w_i:=w_{i-1};~\mask_i=0 \}\} % \lfloor
\end{align*}
%(Replacing $γ^i$ by $v_i$ would work as well.)
By induction we can show that $0≤w_i≤γ^i$. This holds for $w_0$. 
Assume $0≤w_{i-1}≤γ^{i-1}$: If $w_{i-1}<γ^i$ then $w_i=w_{i-1}<γ^i$.
\begin{align*}
  \text{If $w_{i-1}≥γ^i$ then $w_i=w_{i-1}-v_i≤γ^{i-1}-γ^{i+1}=(γ^{-1}-γ)γ^i≤γ^i$.}
\end{align*}
The last inequality is true for $γ≥\frac12(\sqrt{5}-1)$, which is satisfied due to the restriction $1/φ≤γ<1$.
Hence the error $0≤w-\hat w=w_k≤γ^k≤ε≤γ^{k-1}$ for $k:=\lceil\log_γ ε\rceil\geq0$.

% Instead of assuming $v_i∈I_i$, 
Now consider a random sequence $(v_{i})_{i\in[\msim]}$ where
we sample $v_s \overset{iid}{\sim} P$ over the interval $[γ^2 ε,γ]$ for $s=1,...,\msim>k$.
% If there is at least one sample in each interval $I_i$, we can set $b$ for this sample to $\mask_i$ constructed above, and all other $b$ to zero, resulting in the desired approximation.
In the event that there is at least one sample in each interval $I_i$, we can use the construction above with a subsequence $\tilde{v}$ of $v$
such that $\tilde{v}_i\in I_i$ 
and $\sum_{i\in[k]} \mask_i \tilde{v}_i =w_k$ as in the construction above.
%$(v_{j_i})_{j_i\in[k]}$ such that there is one $v_{j_i}$ in each interval $I_i$, with $\mask_i$'s for this subsequence as constructed above, while setting all other $\mask_i$'s to zero.
%
Next we lower bound the probability $p$ that each interval $I_i$ contains at least one sample.
Let $E_i$ be the event ``no sample is in $I_i$''
and let $c=\min_{i\in[k]} P[v∈I_i]$. 
Then 
$P[E_i]=(1-P[v∈I_i])^\msim≤(1-c)^\msim$, hence
\begin{align*}
  p ~=~ 1-P[E_1∨...∨E_k] ~≥~ 1-\sum_{i=1}^kP[E_i] ~≥~ 1-k(1-c)^\msim ~≥~ 1-k\exp(-c\msim)
\end{align*}
and thus choosing $\msim \geq \ceiling{\frac1c\ln(k/δ)}$
ensures that $p \geq 1-δ$.
Finally,
\begin{align}\label{eq:cnorm}
    \! c = \min_{i\in[k]} P[v∈I_i] = \min_i P[γ^{i+1}\!<v\!≤\!γ^i] = \min_i \int_{γ^{i+1}}^{γ^i}\frac{c'}v dv = c'\ln\frac{1}{γ} = 1/\log_γ(γε) = \frac{1}{k'}
\end{align}
and so we can take $\msim = \ceiling{k'\ln\frac{k'}{δ}}$.
%Finally we determine $c$ for the particular choice
%of $P(v)=c'/v$ for $v∈[γ^2 ε,γ]$, where $c':=1/\ln(1/εγ)$ correctly normalizes $\int\frac{c'}v dv=1$:
%\begin{align*}
%  c ~\leq~ P[v∈I_i] ~=~ P[γ^{i+1}<v≤γ^i] ~=~ \int_{γ^{i+1}}^{γ^i}\frac{c'}v dv ~=~ c'\ln\frac{1}{γ} 
%\end{align*} %\qed
%so it suffices to choose $c ~:=~ c'\ln\frac{1}{γ}$.
\end{proof}

%%%%%%%%%%%%%%%%%%%%%%%%%%%%%%%%%%%%%%%%%%%%
%%%%%%%%%%%%%%%%%%%%%%%%%%%%%%%%%%%%%%%%%%%%
%%%%%%%%%%%%%%%%%%%%%%%%%%%%%%%%%%%%%%%%%%%%

\begin{corollary}[Golden-ratio decomposition for weights in $[0, \wmax{]}$]
\label{cor:samplemax} % all theorems/lemmas/etc should have a title
  For any given $\eps>0$, define the probability density
  $P_{\eps}(v):=\frac{c'}{v}$ for $v \in [\frac49ε,\frac23\wmax]$ 
  with normalization $c':=1/\ln\frac{3\wmax}{2ε}$.
  Let $k':=\log_{\nicefrac{3}{2}}\frac{3\wmax}{2ε}$,
  For any $\delta\in(0,1)$,
  if $\msim=\lceil k'\ln\frac{k'}{δ}\rceil=\tilde{Ω}(\ln\frac1{ε}⋅\ln\frac1{δ})$,
  then with probability at least $1-δ$ over the random sampling of $\msim$ `weights' $v_s \sim P_{\eps}$ ($s=1,...,\msim$) the following holds:
  For every target `weight' $w∈[0,\wmax]$,
  there exists a mask $\vmask ∈\{0,1\}^\msim$ with $|\vmask|≤ k'$
  such that $\hat w:=\mask_1 v_1+...+\mask_\msim v_\msim$ is $ε$-close to $w$,
  indeed $w-ε≤\hat w≤w$.
\end{corollary}
\begin{proof}
Follows from \cref{lem:grd} with $γ=2/3$ and a simple rescaling argument:
First rescale $w' = w/\wmax$ and apply \cref{lem:grd} with $w'$ and accuracy $\eps/\wmax$.
Then the constructed $\hat{w}'$ satisfies
$w'-ε/\wmax≤\hat{w}'≤w'$
%$\hat{w}' \leq w' \leq \hat{w}' + \eps/\wmax$
and multiplying by $\wmax$ gives the required accuracy.
Also note that the density $P_\eps(v)\propto 1/v$ is scale-invariant.
% L: What's the definition of 'scaling linearly' then? 
\end{proof}
\vspace{1mm}

%%%%%%%%%%%%%%%%%%%%%%%%%%%%%%%%%%%%%%%%%%%%%%%%%%%%%%%%%%%%%%%
\section{Related Work}
%%%%%%%%%%%%%%%%%%%%%%%%%%%%%%%%%%%%%%%%%%%%%%%%%%%%%%%%%%%%%%%

In the version of this paper that was submitted for review,
we conjectured with supporting experimental evidence that high precision could be obtained also with uniform sampling
when taking advantage of sub-sums (see \cref{sec:uniform_sums}).
After the submission deadline, we have been made aware that \citet{pensia2020optimal}
concurrently and independently submitted a paper that resolves this conjecture, by using a theorem of \citet{lueker1998partition}.
\citet{pensia2020optimal} 
furthermore use a different grouping of the samples in each layer, 
leading to a refined bound with a logarithmic dependency on the number of \emph{weights} per target weight and provide a matching lower bound (up to constant factors).
Their results are heavily anchored in the assumptions that the max norms of the weights and of the inputs are bounded by 1, leaving open the question of what happens without these constraints---this could be dealt with by combining their results with our \cref{thm:propagation}.

%%%%%%%%%%%%%%%%%%%%%%%%%%%%%%%%%%%%%%%%%%%%%%%%%%%%%%%%%%%%%%%
\section{Conclusion}
%%%%%%%%%%%%%%%%%%%%%%%%%%%%%%%%%%%%%%%%%%%%%%%%%%%%%%%%%%%%%%%

We have proven that large randomly initialized ReLU networks contain many more subnetworks
than previously shown, which gives further weight to the idea that one important task
of stochastic gradient descent (and learning in general) may be to effectively prune connections
by driving their weights to 0, revealing the so-called winning tickets.
One could even conjecture that the effect of pruning is to reach a vicinity of the global 
optimum, after which gradient descent can perform local quasi-convex optimization.
Then the required precision $\eps$ may not need to be very high.

\iffalse
Although our work assumes that the distribution of the weights is hyperbolic,
we conjecture (with supporting empirical evidence in \cref{sec:uniform_sums})
that a similar effect could be achieved with uniform samples.
Combined with some of our proof ideas (batch sampling and product weights),
it may be possible to reach the same type of bounds.
\fi

Further questions include the impact of convolutional and batch norm layers, 
skip-connections and LSTMs on the number of required sampled neurons to maintain a good accuracy.

\ifarxiv
\else
\section*{Statement of broader impact}

This work is theoretical, and in this regard we do not expect any direct societal or ethical consequences. It is our hope, however, that by studying the theoretical foundations of neural networks this will eventually help the research community make better and safer learning algorithms.

\comment{Raia said on Twitter than for theory work we could even say 'not applicable', so this short statement should be enough? We may actually use this opportunity to give us a small bonus though.}
\fi

\subsubsection*{Acknowledgements}
The authors would like to thank Claire Orseau for her great help with time allocation, Tor Lattimore for the punctual help and insightful remarks, Andr\'as Gy\"orgy for initiating the Neural Net Readathon, and Ilja Kuzborskij for sharing helpful comments.

\iffalse
%%%%%%%%

In order to provide a balanced perspective, authors are required to include a statement of the potential broader impact of their work, including its ethical aspects and future societal consequences. Authors should take care to discuss both positive and negative outcomes.

%%%=== From the NeurIPS2020 example file
Authors are required to include a statement of the broader impact of their work, including its ethical aspects and future societal consequences. Authors should discuss both positive and negative outcomes, if any. For instance, authors should discuss a) who may benefit from this research, b) who may be put at disadvantage from this research, c) what are the consequences of failure of the system, and d) whether the task/method leverages biases in the data. If authors believe this is not applicable to them, authors can simply state this.

Use unnumbered first level headings for this section, which should go at the end of the paper. {\bf Note that this section does not count towards the eight pages of content that are allowed.}

See \url{https://docs.google.com/document/d/1rFXCrPR3qdhzHJ4fX0Mol65eqDJQq3DQIJRSElczoZM/edit#heading=h.ppzewsm3wwcx}
\fi

%-------------------------------%
%\paradot{BibTeX}
%-------------------------------%
\bibliographystyle{abbrvnat}
\bibliography{biblio.bib}

\begin{thebibliography}{14}
\providecommand{\natexlab}[1]{#1}
\providecommand{\url}[1]{\texttt{#1}}
\expandafter\ifx\csname urlstyle\endcsname\relax
  \providecommand{\doi}[1]{doi: #1}\else
  \providecommand{\doi}{doi: \begingroup \urlstyle{rm}\Url}\fi

\bibitem[Allen-Zhu et~al.(2019)Allen-Zhu, Li, and
  Song]{allen-zhu2019convergence}
Z.~Allen-Zhu, Y.~Li, and Z.~Song.
\newblock A convergence theory for deep learning via over-parameterization.
\newblock In \emph{International Conference on Machine Learning}, pages
  242--252, 2019.

\bibitem[Ebendt and Drechsler(2009)]{ebendt2009weighted}
R.~Ebendt and R.~Drechsler.
\newblock Weighted {A∗} search – unifying view and application.
\newblock \emph{Artificial Intelligence}, 173\penalty0 (14):\penalty0 1310 --
  1342, 2009.

\bibitem[Frankle and Carbin(2019)]{frankle2018lottery}
J.~Frankle and M.~Carbin.
\newblock The lottery ticket hypothesis: Finding sparse, trainable neural
  networks.
\newblock In \emph{ICLR}, 2019.

\bibitem[Khan et~al.(2020)Khan, Sohail, Zahoora, and Qureshi]{khan2020convnet}
A.~Khan, A.~Sohail, U.~Zahoora, and A.~S. Qureshi.
\newblock A survey of the recent architectures of deep convolutional neural
  networks.
\newblock \emph{Artificial Intelligence Review}, Apr. 2020.

\bibitem[Livni et~al.(2014)Livni, Shalev-Shwartz, and Shamir]{livni2014}
R.~Livni, S.~Shalev-Shwartz, and O.~Shamir.
\newblock {On the computational efficiency of training neural networks}.
\newblock In \emph{Advances in neural information processing systems}, pages
  855--863, 2014.

\bibitem[Lueker(1998)]{lueker1998partition}
G.~Lueker.
\newblock Exponentially small bounds on the expected optimum of the partition
  and subset sum problems.
\newblock \emph{Random Structures and Algorithms}, 12:\penalty0 51--62, 1998.

\bibitem[Ma et~al.(2018)Ma, Bassily, and Belkin]{belkin-late2018power}
S.~Ma, R.~Bassily, and M.~Belkin.
\newblock The power of interpolation: Understanding the effectiveness of sgd in
  modern over-parametrized learning.
\newblock In \emph{International Conference on Machine Learning}, pages
  3325--3334, 2018.

\bibitem[Malach et~al.(2020)Malach, Yehudai, Shalev-Shwartz, and
  Shamir]{malach2020proving}
E.~Malach, G.~Yehudai, S.~Shalev-Shwartz, and O.~Shamir.
\newblock Proving the lottery ticket hypothesis: Pruning is all you need.
\newblock \emph{arXiv preprint arXiv:2002.00585. To appear in ICML-2020}, 2020.

\bibitem[Pensia et~al.(2020)Pensia, Rajput, Nagle, Vishwakarma, and
  Papailiopoulos]{pensia2020optimal}
A.~Pensia, S.~Rajput, A.~Nagle, H.~Vishwakarma, and D.~Papailiopoulos.
\newblock Optimal lottery tickets via subsetsum: Logarithmic
  over-parameterization is sufficient.
\newblock \emph{arXiv preprint arXiv:2006.07990. To appear in NeurIPS-2020},
  2020.

\bibitem[Ramanujan et~al.(2019)Ramanujan, Wortsman, Kembhavi, Farhadi, and
  Rastegari]{ramanujan2019s}
V.~Ramanujan, M.~Wortsman, A.~Kembhavi, A.~Farhadi, and M.~Rastegari.
\newblock What's hidden in a randomly weighted neural network?
\newblock \emph{arXiv preprint arXiv:1911.13299}, 2019.

\bibitem[Schrittwieser et~al.(2019)Schrittwieser, Antonoglou, Hubert, Simonyan,
  Sifre, Schmitt, Guez, Lockhart, Hassabis, Graepel, Lillicrap, and
  Silver]{schrittwieser2019mastering}
J.~Schrittwieser, I.~Antonoglou, T.~Hubert, K.~Simonyan, L.~Sifre, S.~Schmitt,
  A.~Guez, E.~Lockhart, D.~Hassabis, T.~Graepel, T.~Lillicrap, and D.~Silver.
\newblock Mastering atari, go, chess and shogi by planning with a learned
  model.
\newblock \emph{arXiv preprint arXiv:1911.08265}, 2019.

\bibitem[Ulyanov et~al.(2018)Ulyanov, Vedaldi, and Lempitsky]{ulyanov2018}
D.~Ulyanov, A.~Vedaldi, and V.~Lempitsky.
\newblock Deep image prior.
\newblock In \emph{Proceedings of the IEEE Conference on Computer Vision and
  Pattern Recognition}, pages 9446--9454, 2018.

\bibitem[van~den Oord et~al.(2016)van~den Oord, Dieleman, Zen, Simonyan,
  Vinyals, Graves, Kalchbrenner, Senior, and Kavukcuoglu]{oord2016wavenet}
A.~van~den Oord, S.~Dieleman, H.~Zen, K.~Simonyan, O.~Vinyals, A.~Graves,
  N.~Kalchbrenner, A.~Senior, and K.~Kavukcuoglu.
\newblock Wavenet: A generative model for raw audio, 2016.

\bibitem[Zou and Gu(2019)]{zou2019improved}
D.~Zou and Q.~Gu.
\newblock An improved analysis of training over-parameterized deep neural
  networks.
\newblock In \emph{Advances in Neural Information Processing Systems}, pages
  2053--2062, 2019.

\end{thebibliography}

\newpage
\begin{appendices}
\counterwithin{theorem}{section}

%% Must be included
%%%%%%%%%%%%%%%%%%%%%%%%%%%%%%%%%%%%%%%%%%%%%%%%%%%%%%%%%%%%%%%
\section{Sub-sums of Uniform Samples}\label[appendix]{sec:uniform_sums}
%%%%%%%%%%%%%%%%%%%%%%%%%%%%%%%%%%%%%%%%%%%%%%%%%%%%%%%%%%%%%%%

\begin{figure}
    \centering
    \includegraphics[width=0.75\textwidth]{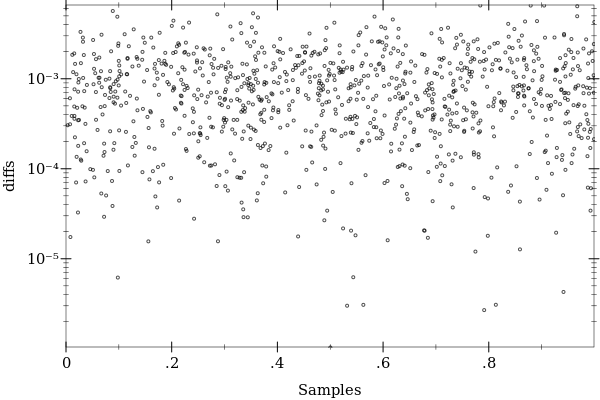}
    \includegraphics[width=0.75\textwidth]{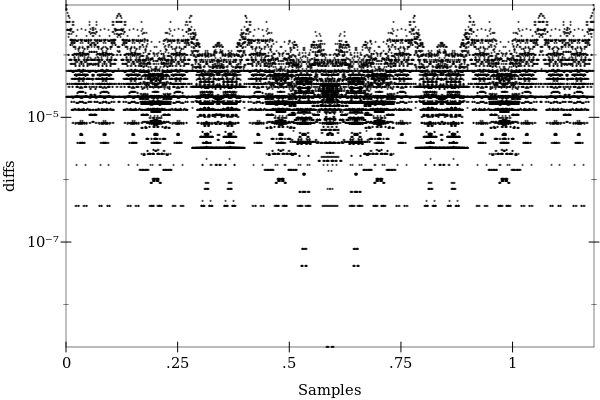}
    \includegraphics[width=0.75\textwidth]{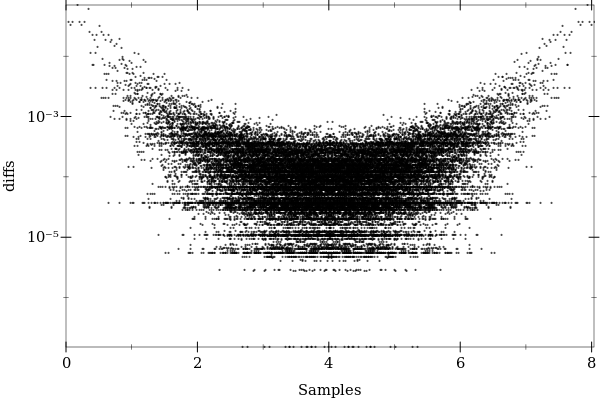}
%    \parbox{0.33\textwidth}{
%        \includegraphics[width=0.33\textwidth]{img/1000_uniform.png}}
%    \parbox{0.32\textwidth}{
%        \includegraphics[width=0.33\textwidth]{img/15-hyperbolic.png}}
%    \parbox{0.33\textwidth}{
%        \includegraphics[width=0.33\textwidth]{img/15-uniform-sums.png}}
    \caption{y-axis: Difference between two consecutive points on the x-axis.
    Top: 1000 uniform samples in [0, 1], x-axis is sample value (previous work);
    Middle: 15 hyperbolic samples in [0, 1], each point is one of $2^{15}$ possible sub-sums,
    x-axis is the sub-sums (this work);
    Bottom: Like Middle, but with uniform samples in [0, 1] (future work?).}
    \label{fig:densities}
\end{figure}

Sampling uniformly requires many samples to obtain high precision anywhere in the region of interest
(\cref{fig:densities}, Top).
In this work we have taken advantage the summing function in neurons,
combined with pruning,
so as to be able to consider all $2^k$ sub-sums of $k$ samples
(\cref{fig:densities}, Middle).
However, we conjecture that a similar effect appears with sub-sums of uniform samples
(\cref{fig:densities}, Bottom, but observe the large offset): For example,
it suffices that 2 among $k$ samples $x_1$ and $x_2$ are within $ε$ of each other
so that for all other samples $x_3$, $x_3+x_1$ and $x_3+x_2$ are within $ε$ of each other too.

%%%%%%%%%%%%%%%%%%%%%%%%%%%%%%%%%%%%%%%%%%%%%%%%%%%%%%%%%%%%%%%
\section{Variations and Improvements}\label[appendix]{sec:improvements}
%%%%%%%%%%%%%%%%%%%%%%%%%%%%%%%%%%%%%%%%%%%%%%%%%%%%%%%%%%%%%%%

\begin{remark}
Sampling from $P_\eps$ is easy using the inverse CDF:
To obtain a sample from $P_\eps$ in $[α, β]$, first draw a uniform sample $u\sim\mathcal{U}[0, 1]$
then return $α(β/α)^u$.
\end{remark}

\begin{remark}
The sampling procedure can be made independent of $ε$ by sampling from 
$P(v)=\ln\frac1{γ}/v(\ln v)^2$ for $0≤v≤γ$ with $γ=\frac23$. 
It is easy to see that $c$ in the proof becomes $(\lnγ/\ln(εγ^2))^2$,
leading to a slightly worse bound $\msim=\tilde{Ω}((\ln\frac1{ε})^2⋅\ln\frac1{δ})$.
For $P(v)\propto [v\ln\frac1{v}(\ln(\ln\frac3{v}))^2]^{-1}$ we get the same bound 
$m=\tilde{Ω}(\ln\frac1{ε}⋅\ln\frac1{δ})$ as in Lemma~\ref{lem:grd}.
\comment{If you want good numerical results, you need to optimize the constants
$a$ and $b$ in $P(v)\propto 1/v(\ln\frac{a}{v})^2$ and $P(v)\propto [v\ln\frac{a}{v}(\ln(\ln\frac{b}{v}))^2]^{-1}$.
Inspecting the graph of the density, maybe $a≈5$ and $b≈15$ are good for interesting ranges of $ε$?}
\end{remark}

\begin{remark}
\XXX
Instead of using batch sampling and \cref{lem:grdM}, we can 
`recycle' samples in a different way, which removes the leading factor 16 
at the expense of a larger second order term.
See \cref{sec:sample_recycling}.
\end{remark}

\begin{remark}
What if $\wmax$ is not known?
A simple trick is to $\wmax=1$ every second sample, every 4th sample we take $\wmax=2$, and every $2^{j+1}$ sample we can take $\wmax=2^j$. Then the total number of samples required   
to obtain $M^*_i$ samples for the correct $\wmax^*$ is at most $4\wmax^* M^*_i$.
Since $\wmax^*$ is expected to be small, this is likely a mild dependency.
A $\log\wmax\log^2 \log\wmax$ factor instead of $\wmax$ can even be obtained with a little more work.
\XXX
\comment{This is not optimal though.
We should use GBS, define a prior over the integers $i$, then take $\wmax=2^i$.
Gives immediately a bound of $O(2M^*/p(i) )$
}
\end{remark}

\begin{remark}
In practice, weights are often initialized uniformly in 
$[-O(\sqrt{1/n}),+O(\sqrt{1/n})]$, where $n$ is the layer width, potentially somehow averaged over two layers, 
i.e.\ weights are initially very small. 
All our initializations need \emph{some} large weights but only very few % ($\frac12\log(nw_{\max})+O(1)$ 
($O(\log n)$ outside this interval), 
most weights are very small too. 
We could even eliminate the large weights and limit our sampling procedure to weights in this interval,
but sample $O(\sqrt{n})$ times more weights to 
reconstruct large weights.
\end{remark}

\begin{remark}
Empirically it seems that even sampling a logarithmic number of weights uniformly from $[-1;+1]$ 
(or $[±1/\sqrt{n}]$, see previous item) or from a standard Gaussian works nearly as well hyperbolic sampling, 
but we were not able to prove this.
\todo{See appendix?}
\comment{L: I'm not sure it's enough to have log bounds though, it may still be polynomial
but better than mere $1/ε$}
\comment{M: I am reasonably confident that it's polylog in $1/ε$, maybe $O(\lnε)^2$.
I spent quite some time on it. That's the problem if you work on the hard problems, that often there is nothing to show for.}
\end{remark}

%%%%%%%%%%%%%
\iffalse
\begin{remark}
\item The (deterministic and stochastic) 
constructions so far need multiple weights between the same pair of neurons,
namely $m$ into an intermediate neuron and 2 out of it.
The intermediate linear layer can be eliminated by
directly connecting two neurons by $2m$ signed weights
$-2^{-1},...,-2^{-m},+2^{-1},...,+2^{-m}$, 
but actually we do not recommend this, because it leads to $2m$ connections instead of $m+2$.
\end{remark}
\begin{remark}

\item More speculative, with logistic sigmoid activation (but not most others),
at least for the intermediate layer,
(wide) uniform (or Gaussian) sampling of the weights can be easier proven to lead to the same result,
since the output of a sigmoid with larger negative input is approximately an exponential function,
and exponentiated uniform noise is hyperbolically distributed,
the only property the proof of the Lemma above relies on.
\end{remark}
\begin{remark}
What if the target function is not given as a network?
\comment{Mehrdad's question}
First find a network that computes it, second approximate it.
Note that we can compete with the `nicest' network in the target class (particularly important 
when the activation function is discontinuous).
\end{remark}
\fi
%%%%%%%%%%%%%

\begin{remark}
\XXX
Extreme case: Pruning `Boolean' networks
The difficulty of pruning can be seen easily 
for Boolean network with Heaviside transition function and binary inputs, and Boolean weights everywhere.
Then network $G$ would need only twice as many weights as the target network, 
but can still represent exponentially many functions.
It is then clear that not only ``Pruning really is all you need,'' but also that
``Pruning is as hard as learning.''

\todo{Result showing that learning Boolean functions is at least NP-hard}

\comment{O: I (omar) think this is a very (!) interesting thing to tell, that pruning can be as hard as learning. Does it imply (or at least suggest) that for very big networks the combinatorial problem of searching sparse subnetworks will be out of reach?}
\comment{\citet{malach2020proving} already make this point and \citet{ramanujan2019s} have some experimental results}

\todo{L: Can we provide a simple pruning algorithm based on training example 
and show convergence speed? Will likely be exponential though.}
\end{remark}

%%%%%%%%%%%%%%%%%%%%%%%%%%%%%%%%%%%%%%%%%%%%%%%%%%%%%%%%%%%%%%%
\section{Technical Results}\label[appendix]{sec:technical}
%%%%%%%%%%%%%%%%%%%%%%%%%%%%%%%%%%%%%%%%%%%%%%%%%%%%%%%%%%%%%%%

\textbf{Proof of \cref{thm:propagation}.}
\propagproof

%%%%%%%%%%%%%%%%%%%%%%%%%%%%%%%%%%%%%%%%
\iffalse
\begin{lemma}[sub-multiplicativity of the matrix 2-norm]
\label{lem:submult}
For all matrices $A \in \Reals^{n\times m}$ and $B \in \Reals^{m\times p}$:
\begin{align*}
    \|AB\|_2 \leq \|A\|_2 \|B\|_2\,.
\end{align*}
\end{lemma}
\begin{proof}
Using the Cauchy-Schwarz inequality,
\begin{align*}
    \| AB \|_2^2
&= \sum_{i}\sum_{j} |(AB)_{i,j}|^2
= \sum_{i}\sum_{j} \left|\sum_{k} A_{i,k}B_{k,j}\right|^2 \\
&\leq \sum_{i}\sum_{j} \left(\sum_{k} A_{i,k}^2\right)\left(\sum_{k} B_{k,j}^2\right)
= \left(\sum_{i}\sum_{k} A_{i,k}^2\right)\left(\sum_{j}\sum_{k} B_{k,j}^2\right)
= \| A \|_2^2 \| B \|_2^2
\end{align*}
and the result is proven by taking the square root on both sides.
%Using the Cauchy-Schwarz inequality twice:
%\begin{align*}
%    \|AB\|_2 &= \| [A_{i,*} B_{*, i}]_{i\in[n]} \|_2 
%    \leq 
%    \left\| [\|A_{i,*}\|_2\, \|B_{*,i}\|_2]_{i\in[n]} \right\|_2  \\
%    &\leq \left\|[\|A_{i, *}\|_2]_{i\in[n]}\right\|_2\, \left\|[\|B_{*, %i}\|_2]_{i\in[n]}\right\|_2
%    = \|A\|_2 \|B\|_2
%\end{align*}/
%which proves the result.
%\comment{Do we need to transpose $B_{*,i}$?}
\end{proof}
\fi
%%%%%%%%%%%%%%%%%%%%%%%%%%%%%%%%%%%%%%%%

%%%%%%%%%%%%%%%%%%%%%%%%%%%%%%%%%%%%%%%%
\begin{lemma}[Bound on positive sequences]\label{lem:ax+b_general}
Assuming $x_0 \geq 0$,
and if, for all $t=0, 1, \ldots$, 
$x_{t} \leq a_t x_{t-1} + b_t$
with $a_t \geq 0$ and $b_t \geq 0$, then 
\begin{align*}
\forall \tau &\text{ s.t. } |\{a_t < 1+1/\tau\}_{t\in[T]}| ~\leq~ \tau\,\quad\text{we have}\quad
x_{T} ~\leq~ e\left(x_0 + c\right)\prod_{\substack{t\in[T]\\a_t\geq 1+1/\tau}}^T a_t\,\\
\text{with } c &=  \min\left\{\tau\max_t b_t, ~\max_t \frac{b_t}{a_t-1}\right\}\,.
%x_{T} \leq -c +e\left(x_0 + c\right)\prod_{\substack{t\in[T]\\a_t\geq 1+1/\tau}}^T a_t \\
\end{align*}
\end{lemma}
\begin{proof}
First, observe that $\tau\geq 0$, with $\tau=0$ iff $T=0$.   
Let $\tilde{a}_t = \max\{a_t, 1+1/\tau\}$.
Then
\begin{align*}
    c = \min\left\{\tau\max_t b_t, ~\max_t \frac{b_t}{a_t-1}\right\}
    = \max_t \frac{b_t}{\max\{1/\tau, a_t-1\}}
    = \max_t \frac{b_t}{\tilde{a}_t-1}\,.
\end{align*}
Define $y_{t} = \tilde{a}_t y_{t-1} + (\tilde{a}_t-1)c$ and $y_0 = x_0$.
Then we have $y_{t} + c = \tilde{a}_t \left(y_{t-1} + c\right)$
and so by recurrence $y_{T} + c = (y_0 + c)\prod_{t=1}^T ~a_t$
and thus $y_{T} \leq (y_0 + c)\prod_{t=1}^T \tilde{a}_t$.
Now, observe that $y_{t} \geq \tilde{a}_t y_{t-1} + b_t \geq 0$ and so by recurrence with base case $y_0 = x_0$, $x_{T} \leq y_{T}$.

Furthermore 
\begin{align*}
    \prod_{t\in[T]} \tilde{a}_t ~\leq~
    \prod_{t: a_t < 1+1/\tau}(1+1/\tau) \prod_{t: a_t \geq 1+1/\tau} a_t 
    ~\leq~ (1+1/\tau)^\tau \prod_{t: a_t \geq 1+1/\tau} a_t\,,
\end{align*}
noting that $(1+1/\tau)^\tau\leq e$.
\end{proof}

\begin{remark}
The factor $e$ should be 1 if $\min_t a_t \geq 1+1/\tau$.
\end{remark}
\begin{remark}
$\tau=T$ is always feasible.
\end{remark}
\begin{remark}
If $\min_t a_t \geq 2$, then $\tau=1$ is feasible.
\end{remark}

\begin{corollary}[Feasible $\tau$ for \cref{lem:ax+b_general}]\label{cor:feasibletau}
In the context of \cref{lem:ax+b_general},
for all $x>1$
taking $\tau = \max\{1/(x-1),~|\{a_t < x\}_{t\in[T]}|\}$ is feasible.
\end{corollary}
\begin{proof}
Take $y=1/(x-1)$, so $x=1+1/y$ and $\tau = \max\{y,~ |\{a_t < 1+1/y\}_t|\}$.
Thus $|\{a_t < 1+1/\tau\}_t| \leq |\{a_t < 1+1/y\}_t| \leq \tau$ as required.
\end{proof}
\begin{remark}
If $1 < \min_t a_t \leq x$ then taking $\tau = 1/(x-1)$ is feasible.
\end{remark}
\begin{remark}
Taking $\tau=\max\{1, |\{a_t < 2\}_t|\}$ is feasible.
\end{remark}
\begin{remark}
For $\varphi=(1+\sqrt{5})/2\leq 1.62$,
taking $\tau=\max\{\varphi, |\{a_t < \varphi\}_t|\}$ is feasible.
\end{remark}
\begin{remark}
If $\min_t a_t > 1$ then $\tau = 1/(\min_t a_t-1)$ is feasible (but useful only if $\min_t a_t \geq 1+1/T$).
\end{remark}

%%% This is amazingly loose, but that's still okay since we're going to take the log.
\begin{corollary}[Simpler bound on positive sequences]\label{lem:ax+b_protected}
Assuming $x_0 = 0$,
and if, for all $t=0, 1, \ldots$, 
$x_{t} \leq a_t x_{t-1} + b_t$
with $a_t \geq 0$ and $b_t \geq 0$, then 
\begin{align*}
x_{T} \leq eT\max_t b_t\prod_{t\in[T]}^T \max\{1, a_t\}\,.
\end{align*}
\end{corollary}
\begin{proof}
Follows from \cref{lem:ax+b_general} with $\tau=T$ which is always feasible
and observing that $\prod_{t:a_t\geq 1+1/T} a_t \leq \prod_t \max\{1, a_t\}$,
and that $c\leq \tau\max_t b_t = T\max_t b_t$.
\end{proof}

%%%%%%%%%%%%%%%%%%%%%%%%%%%%%%%%%%%%%%%%

\begin{lemma}[Product of weights]\label{lem:prodw}
Let probability densities $P_v(v):=c/v$ for $v∈[a,b]$
and $0<a<b$ with normalization $c:=1/\ln\frac{b}{a}$.
Let weight $w:=v⋅v'$ with $v$ and $v'$ both sampled from $P_v$.
Then $P_w(w)≥c/2w$ for $w∈[a',b']$,
where $P_w$ is the probability density of $w$, 
and $a':=a\sqrt{ab}$ and $b':=b\sqrt{ab}$.
\end{lemma}

Note that $w$ may be outside of $[a',b']$, but at least half of the time is inside $[a',b']$.
The lemma implies that the bound in \cref{lem:grd} also applies to the product of two weights,
only getting a factor of 2 worse. Note that the scaling ranges $\frac{b}{a}=\frac{b'}{a'}$ are the same.

\begin{proof}
$\ln v$ is uniformly distributed in $[\ln a, \ln b]$: 
indeed, taking $y=\ln v$, we have $P_y(y)=P_v(v)/\frac{dy}{dv} = c$.
Let us scale and shift this to $t:=c(2\ln v-\ln ab)∈[-1, +1]$ 
and similarly $t':=c(2\ln v'-\ln ab)∈[-1, +1]$.
Then $P_t(t)=P_v(v)/\frac{dt}{dv}=\frac{c}{v}/\frac{2c}{v}=\frac12$ for $t∈[-1, +1]$,
and same for $t'$. 
Let $u:=t+t'∈[-2,+2]$. 
The sum of two uniformly distributed random variables is triangularly distributed:
$P_u(u)=\frac12(1-\frac12|u|)$.
Using $w=v⋅v'$ we can write $u=t+t'=2c(\ln w-\ln(ab))$.
Then $P_w(w)=P_u(u)\frac{du}{dw}=\frac12(1-\frac12|u|)\frac{2c}{w}$.
For $|u|≤1$ this is $≥c/2w$. 
Finally $|u|≤1$ iff $|\ln w-\ln(ab)|≤1/2c$ iff $\ln w\gtreqless\ln ab∓\frac12\ln\frac{b}{a}$ iff $w∈[a',b']$.
\end{proof}

%%%%%%%%%%%%%%%%%%%%%%%%%%%%%%%%%%%%%
%% With γ
\iffalse
\begin{lemma}[Product of weights]\label{lem:prodw}
Let probability densities $P_v(v):=c/v$ for $v∈[a,b]$ and $P_{v'}(v'):=c/v'$ 
for $v'∈[γa,γb]$ and $b>a>0<γ$ with normalization $c:=1/\ln\frac{b}{a}$.
Let weight $w:=v⋅v'$ with $v$ and $v'$ sampled from $P_v$ and $P_{v'}$ respectively. 
Then $P_w(w)≥c/2w$ for $w∈[a',b']$,
where $P_w$ is the probability density of $w$, 
and $a':=γ\sqrt{a^3b}$ and $b':=γ\sqrt{ab^3}$.
\end{lemma}

Note that $w$ may be outside of $[a',b']$, but at least half of the time is inside $[a',b']$.
The lemma implies that the bound in Lemma~\ref{lem:sample} also applies to the product of two weights,
only getting a factor of 2 worse. Note that the scaling ranges $\frac{b}{a}=\frac{γb}{γa}=\frac{b'}{a'}$ are the same.

\begin{proof}
$\ln v$ is uniformly distributed in $[\ln a,\ln b]$. 
Lets scale and shift this to $t:=c(2\ln v-\ln ab)∈[-1,+1]$ 
and similarly $t':=c(2\ln v'-\lnγ^2ab)∈[-1,+1]$.
Then $P_t(t)=P_v(v)/\frac{dt}{dv}=\frac{c}{v}/\frac{2c}{v}=\frac12$ for $t∈[-1,+1]$,
and same for $t'$. 
Let $u:=t+t'∈[-2,+2]$. 
The sum of two uniformly distributed random variables is triangularly distributed:
$P_u(u)=\frac12(1-\frac12|u|)$.
Using $w=v⋅v'$ we can write $u=t+t'=2c(\ln w-\lnγab)$.
Then $P_w(w)=P_u(u)\frac{du}{dw}=\frac12(1-\frac12|u|)\frac{2c}{w}$.
For $|u|≤1$ this is $≥c/2w$. 
Finally $|u|≤1$ iff $|\ln w-\lnγab|≤1/2c$ iff $\ln w\gtreqless\lnγab∓\frac12\ln\frac{b}{a}$ iff $w∈[a',b']$.
\end{proof}
\fi
%%%%%%%%%%%%%%%%%%%%%%%%%%%%%%%%%%%%%

%%%%%%%%%%%%%%%%%%%%%%%%%%%%%%%%%%%%%
\begin{lemma}[Filling $k$ categories each with at least $n$ samples]\label{lem:grdM}
Let $P_c$ be a categorical distribution of at least $k\in\Naturals$ (mutually exclusive) categories $\{1, 2,\ldots k, \ldots\}$
such that the first $k$ categories have probability at least $c$ and at most $\nicefrac12$, that is, if $X\sim P_c$, then $c\leq P_c(X=j)\leq \nicefrac12$ for all $j\in[k]$.
Let $(X_i)_{i\in[\Msim]}$ be a sequence of $\Msim$ random variables sampled i.i.d. from $P_c$.
For all $\delta\in(0, 1)$,
for all $n \in \Naturals$,
if
\begin{align*}
    \Msim = \ceiling{\frac2c\left(n + \ln \frac{k}{\delta}\right)}
\end{align*}
then with probability at least $1-\delta$
each category $j\in[k]$
contains at least $n$ samples, \ie $|\{X_i = j\}_{i\leq[\Msim]}| \geq n$.
\end{lemma}
\begin{proof}
Let $c_j\geq c$ be the probability of category $j\in[k]$.
Using the Chernoff-Hoeffding theorem
on the Bernoulli random variable $\indicator{X_i=j}$,
---where $\indicator{test}$ is the indicator function and equals 1 if $test$ is true, 0 otherwise---
with $\Msim c_j-x=n\geq0$, that is, $x=\Msim c_j-n$, 
for each category $j\in[k]$ we have
\begin{align*}
    P\left(\sum_{i=1}^{\Msim} (1-\indicator{X_i=j}) > \Msim(1-c_j)+x\right) &\leq \exp\left(-\frac{x^2}{2\Msim c_j(1-c_j)}\right) \\
    P\left(\sum_{i=1}^{\Msim} \indicator{X_i=j} < \Msim c-x\right) &\leq \exp\left(-\frac{x^2}{2\Msim c_j(1-c_j)}\right) \\
    P\left(\sum_{i=1}^{\Msim} \indicator{X_i=j} < n\right) &\leq \exp\left(-
    (\Msim c_j/2 - n)\right)
\end{align*}
and the condition $(1-c_j)\geq \nicefrac12$ is satisfied. 
Name $E_j$ the event ``the category $j\in[k]$ contains fewer than $n$ samples,'' then 
$P(E_j) \leq \exp(- (\Msim c/2 - n))$.
Then, using a union bound, the probability that any of the $k$ categories contain fewer than $n$ samples is at most
\begin{align*}
    P(E_1\lor E_2\lor\ldots E_k) ~\leq~
    \sum_{j=1}^k P(E_j) ~\leq~
    k\exp\left(- (\Msim c/2 - n)\right)
\end{align*}
and since $\Msim \geq \frac2c\left(n + \ln \frac{k}{\delta}\right)$
\begin{align*}
    P(E_1\lor E_2\lor\ldots E_k) &\leq \delta\,, \\
    1-P(E_1\lor E_2\lor\ldots E_k) &\geq 1-\delta\,,
\end{align*}
which proves the claim.
\end{proof}
%%%%%%%%%%%%%%%%%%%%%%%%%%%%%%%%%%%%%

%% May not be included
\section{Sample recycling}\label[appendix]{sec:sample_recycling}

\begin{theorem}[ReLU sampling bound \#2]\label{thm:relu2}
\cref{thm:relu} holds simultaneously also with 
\begin{align*}
    \Msim_i = \ceiling{2k'\left(n_in_{i-1} + 4\max\{n_i, n_{i-1}\} \ln \frac{2k'\NF}{\delta}\right)}
\end{align*}
and all other quantities are unchanged.
\end{theorem}

\begin{proof}
\textbf{Step 1 and 2.} Same as for \cref{thm:relu}.

\textbf{Step 2'. Sample recycling.}
Let
\begin{align*}
    \msim = \ceiling{8k' \ln \frac{k'}{\delta_w}}\,, \quad
    k' = \log_{\nicefrac32}\frac{3\wmax}{\eps_w}\,,\quad
    k = \ceiling{\log_{\nicefrac32}\frac{2\wmax}{\eps_w}}\,,
\end{align*}
where $\msim$ is the number of neurons that need to be sampled according to \cref{cor:samplemax}
to $\eps_w$-approximate one target weight with probability at least $1-\delta_w$,
and $k$ is an upper bound on the number of unit bits of the corresponding mask.
One neuron with some pruned weights cannot be shared to approximate
two target weights at the same time, which means
we need at least $2kn_in_{i-1}$ neurons ($k$ for each $\hat{w}^+$, and $k$ for each $\hat{w}^-$).
For a specific target weight $\wF$, out of $\msim\geq 2k$ sampled neurons, only at most $2k$ of them are actually used to approximate the target weight;
all others are `discarded'.
But discarding them is wasteful, because only the product weight on the same input/output
as the target weight has been filtered by \cref{cor:samplemax} (via \cref{lem:grd});
all other product weights are still independent samples since their values have not been queried
by any process.
Each intermediate neuron is connected in input and output with $n_i+n_{i-1}$ weights, but it contains exactly only
$\min\{n_i, n_{i-1}\}$ independent \emph{product} weight samples,
since each input weight and each output weight can be used at most as one independent product weight sample.
\cref{alg:recycle} shows that we can use all of them and,
following the algorithm's notation and the assumption that $n_i \geq n_{i-1}$,
that for each $j$ we only need $\msim + 2(k-1)n_{i-1}$ sampled neurons,
that is, only $n_i \msim + 2(k-1)n_{i}n_{i-1}$ for the whole layer.
To also cover the case $n_{i-1} > n_i$, we need to sample $\max\{n_i,n_{i-1}\} \msim + 2(k-1)n_{i}n_{i-1}$ neurons to ensure that every target weight of layer $i$
can be decomposed into $2k$ product weights, each based on $\msim$ independent product weight samples.

\textbf{Step 3. Network approximation.}
For the guarantee to hold simultaneously over all $\hat{w}^+$ and $\hat{w}^-$,
using a union bound we can take $\delta_w = \delta/(2\NF)$.
Finally the claim follows from \cref{thm:propagation} and noting that $k \leq k'$.
\end{proof}

\begin{remark}
Observe that even though the factor in front of $\msim$ is larger than for \cref{thm:relu},
(also $\nlayers\leadsto \NF$ in the log)
we gain a constant factor 8 in front of the leading term $n_in_{i-1}k$.
\end{remark}

\begin{example}
Under the same conditions as \cref{ex:numeric}, \cref{thm:relu2} gives $\Msim_i/\nmax^2\leq 144$,
and under the same conditions as \cref{ex:numeric_free} we have $\Msim_i/\nmax^2\leq 574$.
\end{example}

Therefore, since both \cref{thm:relu} and \cref{thm:relu2} hold simultaneously, we can take:
\begin{align*}
    M_i = \min&\left\{\ceiling{16k'\left(n_in_{i-1} + \ln\frac{2k'\nlayers}{\delta}\right)},\right.  \\
                &\left.\ \ \ceiling{2k'\left(n_in_{i-1} + 4\max\{n_i, n_{i-1}\}\ln \frac{2k'\NF}{\delta}\right)}\right\}
\end{align*}
to ensure that, with probability at least $1-\delta$,
\begin{align*}
    \sup_{x\in\inputset}\|F(x)-\hat{G}(x)\|_2 \leq \eps\,.
\end{align*}

\begin{algorithm}
\begin{lstlisting}
# Sample recycling at layer i.
for j= 1 to $n_i$:  # Assumes $n_i \geq n_{i-1}$
  # Discard old samples and generate fresh ones.
  M = sample $\msim$ fully-connected intermediate neurons
  for d = 1 to $n_{i-1}$:
    # These indices ensure that
    # * all target weights are approximated,
    # * no input weight and no output weight is used for more 
    #   than one target weight.
    idx_in = d
    idx_out = (d+j) % $n_i$
    $\wF = \WF$[idx_in, idx_out]
    # `Call' to the golden-ratio decomposition ((*\cref{cor:samplemax}*))
    # using the provided samples M.
    # It returns the set K $\subseteq$ M of sampled neurons used to decompose $\wF$.
    # Only uses weights at the indices idx_in and idx_out of the neurons in M.
    # The indexes above ensure that no weight in M already has idx_in and
    # idx_out zeroed out.
    K+ = GRD+(M, idx_in, idx_out, $\wF$)  # (*\cref{cor:samplemax}*) for $\hat{w}^+$
    K- = GRD-(M, idx_in, idx_out, $\wF$)  # (*\cref{cor:samplemax}*) for $\hat{w}^-$
    # These samples cannot be reused for other neurons, put them aside.
    M = M \ (K+ $\cup$ K-)

    # Zero-out the input and output weights that the GRD has filtered, 
    # as they are not independent samples anymore and cannot be reused.
    for n in M:
      n.ins[idx_in] = 0
      n.outs[idx_out] = 0
      
    # Fill up M to have $\msim$ intermediate neurons.
    M_new = sample |K+|+|K-| new independent neurons  # |K+|+|K-| $\leq$ 2k
    M = M $\cup$ M_new  # such that |M| = m
\end{lstlisting}
\caption{Recycling samples. We assume that $n_i\geq n_{i-1}$, otherwise the loops and the increments need to be exchanged.}
\label{alg:recycle}
\end{algorithm}

%\input{nonlinear_intermediate_layer.tex}
%\input{flat_region_activation.tex}
%\input{deterministic.tex} % commented out to track reference warnings
%\input{random_discrete.tex}
\iffalse
\input{unused/Lueker98.tex}
\input{unused/lower_bound.tex}
\fi

%% Must be included
\newpage
%%%%%%%%%%%%%%%%%%%%%%%%%%%%%%%%%%%%%%%%%%%%%%%%%%%%%%%%%%%%%%%
\section{List of Notation}\label[appendix]{app:Notation}
%%%%%%%%%%%%%%%%%%%%%%%%%%%%%%%%%%%%%%%%%%%%%%%%%%%%%%%%%%%%%%%

% MH philosophy: For every document, create a list of notation, even if removed
% in the final version, and best before starting to LaTeX your research.
% This gives yourself an overview of all notation
% you use, and makes it easier to improve your notation: to spot
% inconsistencies and to make notation more systematic.
%
% Never use the same symbol for two different things within the
% same paper, not even indices.
% Your co-authors, reviewers, and if kept your readers, will be thankful too.

%% Laurent: Autoformat, see laptop /Prog/Racket/neurips2020_nn/format_notation_table.rkt
\begin{tabular}{ll}
{\bf Symbol }       & {\bf Explanation}                                                                 \\
\hline
%$a/b⋅c=(a/b)⋅c$                 & but $a/bc=a/(bc)$                                                     \\
$\Naturals$                     & natural numbers $\{1, 2, \ldots\}$                                    \\
$\nlayers∈ℕ$                    & number of network layers                                              \\
$\nvec\in \Naturals^\nlayers$   & vector of the number of neurons                                       \\
$\nmax\in\Naturals$             & maximum number of neurons per layer                                   \\
% $i∈ℕ$                         & index of $i$th neuron in layer $\ell$                                 \\
$i∈[\nlayers]$                  & layer index \comment{$\ell∈[L]$ would be more systematic/autological} \\
$j∈[n_i]$                       & index of $j$th neuron in layer $i$                                    \\
%$k∈[n_{i-1}]$                  & index of the $k$th connection of a neuron $(i, j)$                    \\
% $L∈ℕ$                         & number of NN layers                                                   \\
$\Fnn$                          & a target network                                                      \\
$G$                             & the large network to be pruned                                        \\
$\hat{G}$                       & the network $G$ after pruning                                         \\
$\Fmax(\inputset)$              & maximum absolute activation of any non-final neuron on all inputs of interest in $\Fnn$\\
$w∈[-\wmax, \wmax]$             & some weight                                                           \\
$\wmax\in\Reals^+$              & max norm of the weights                                               \\
$\wF∈[-\wmax, \wmax]$           & a weight of the target network $\Fnn$                                 \\
$\WF$                           & weights of the target network                                         \\
$\wa, \wb, \wc, \wD$            & actual individual weights of the network $G$                          \\
$\hat{w}^+, \hat{w}^-, \hat{w}$ & virtual individual weights of the network $\hat{G}$                   \\
$ε>0$                           & output accuracy                                                       \\
$1-δ\in[0, 1]$                  & high probability                                                      \\
$\sigma:ℝ→ℝ$                    & activation function                                                   \\
$\activvec$                     & vector of $\nlayers$ activation functions                             \\
$λ_i$                           & Lipschitz factor of $\activ_i$                                        \\
$k∈ℕ$                           & number of `bits' to represent a weight                                \\
$\msim∈ℕ$                       & number of neurons sampled per target weight                           \\
$\Msim∈ℕ$                       & number of neurons sampled per intermediate layer                      \\
$x∈[-\xmax, \xmax]^{n_0}$       & network input                                                         \\
$\xmax\in\Reals^+$              & max norm of the inputs                                                \\
$P$                             & probability                                                           \\
$A(\nlayers, \nvec, \activvec)$ & architecture of a network                                             \\
$\v v$                          & vector                                                                \\
$\vmask$                        & binary mask vector                                                    \\
$\mask$                         & binary mask                                                           \\
$\Fnn_i$                        & output of layer $i$ of the target network given network inputs        \\
$G$                             & the big network $i$                                                   \\
$\hat{G}_i$                     & subnetwork of the big network, approximating $\Fnn$                   \\
$\fnn_i$                        & layer functions of target network given layer inputs                   \\
$\hat{g}_i$                     & same as $f_i$ for $\hat{G}$                                           \\
$P_\eps$                        & $1/v$ distribution                                                    \\
$\pwplus$                       & $1/v$ distribution                                                    \\
$\pw$                           & $±1/v$ distribution                                                   \\
$\pwprod$                       & $1/v$ product distribution                                            \\
$\pab$                          & product distribution of $\wa$ and $\wb$                               \\
$\pcd$                          & product distribution of $\wc$ and $\wD$                               \\
\end{tabular} % last to be easy to go to on paper

\end{appendices}

\end{document}